\documentclass{article}
\usepackage[preprint]{neurips_2024}

\usepackage[utf8]{inputenc}
\usepackage[T1]{fontenc}
\usepackage{amsmath,amssymb,amsfonts,amsthm}
\usepackage{booktabs}
\usepackage{enumitem}
\usepackage{algorithm}
\usepackage{algpseudocode}
\usepackage{graphicx}
\usepackage{microtype}
\usepackage{url}
\usepackage{hyperref}
\usepackage{natbib}
\usepackage{xcolor}

\hypersetup{
  colorlinks=true,
  linkcolor=blue!45!black,
  citecolor=blue!45!black,
  urlcolor=blue!45!black
}

\newtheorem{theorem}{Theorem}[section]
\newtheorem{proposition}[theorem]{Proposition}
\newtheorem{lemma}[theorem]{Lemma}

\theoremstyle{definition}

\newtheorem{assumption}{Assumption}
\theoremstyle{remark}

\newcommand{\E}{\mathbb E}
\newcommand{\Pbb}{\mathbb P}

\newcommand{\D}{\mathcal D}
\newcommand{\U}{\mathcal U}

\newcommand{\C}{\mathcal C}
\newcommand{\F}{\mathcal F}
\newcommand{\calP}{\mathcal P}
\newcommand{\LB}{\mathrm{LB}}
\newcommand{\UB}{\mathrm{UB}}

\newcommand{\MDE}{\mathrm{MDE}}
\newcommand{\diam}{\mathrm{diam}}
\newcommand{\Rad}{\mathrm{Rad}}
\newcommand{\TV}{\mathrm{TV}}

\title{Privacy-Robust Incrementality Measurement for Advertising Systems under Signal Loss}

\author{%
  {\fontsize{11}{10}\selectfont Prashant Shekhar\thanks{Corresponding author.} and Caroline Howard} \\
  \textit{\fontsize{10}{22}\selectfont Department of Mathematics} \\
  \textit{\fontsize{10}{22}\selectfont Embry-Riddle Aeronautical University} \\
  \textit{\fontsize{10}{22}\selectfont Daytona Beach, FL, USA}
}

\begin{document}
\maketitle

\begin{abstract}
Advertising platforms use randomized lift tests to measure incrementality, but privacy-preserving reporting systems degrade the observed signal through match-rate loss, linkability loss, attribution-window loss, aggregation-threshold suppression, randomized reporting noise, and segment-heterogeneous signal loss. This paper formulates privacy-constrained advertising measurement as a robust causal decision problem under the mentioned signal losses. Given a randomized experiment and an ambiguity set for privacy-induced degradation, the framework projects the observation-compatible fiber of clean/unfiltered experimental worlds onto the incrementality functional and returns certified, rejected, and unresolved decisions. The main result gives a sharp decision frontier. Reports outside the frontier support uniformly valid certification or rejection, whereas reports inside it contain too little information for any method to uniformly distinguish above-threshold incrementality from non-incrementality. Supporting results give finite-sample certification, sample-complexity guarantees, a minimax lower bound showing that signal loss reduces effective information, and a reporting-granularity tradeoff. On 2.0M Criteo Uplift rows and the 64K-row Hillstrom email experiment, clean conversion lift is positive in both datasets, with lifts \(0.00112\) and \(0.00495\), respectively. Population certification survives mild degradation in Criteo and severe degradation in Hillstrom, while all considered finite-sample stress settings in both datasets remain unresolved after simultaneous uncertainty and reporting noise are included. At the segment level, Hillstrom population certification falls from \(100\%\) of coarse cells to \(36.0\%\) of very fine cells, while aggregation-threshold suppression rises to \(39.5\%\). Overall, the research contributes a decision-theoretic layer for privacy-aware incrementality measurement whose output is the strongest causal-claim justified by degraded ads signals.
\end{abstract}

\section{Introduction}

Incrementality measurement asks whether advertising caused outcomes that would not otherwise have occurred. This question is central to ads ranking, advertiser reporting, campaign optimization, budget allocation, and product experimentation. The clean statistical solution is a randomized experiment with well-defined treatment, control, outcomes, and follow-up windows. In practice, however, advertising measurement is increasingly shaped by privacy and identity constraints. Users may opt out of tracking, device identifiers may be missing or unstable, conversions may be reported only in aggregate, small cells may be suppressed, and privacy-preserving systems may add randomized reporting noise before reporting conversion counts. The result is a gap between the causal estimand that the experiment was designed to identify and the degraded signal available to the analyst.

This paper studies that gap. We formulate privacy-constrained advertising measurement as a robust causal decision problem under signal loss. The starting point is a randomized incrementality experiment with a treatment indicator, outcome labels, and covariates. We assume the measurement system releases a degraded version of this experiment through 6 signal-loss layers: (i) match-rate loss drops reportable outcomes, (ii) linkability loss breaks identity continuity across events, (iii) attribution-window loss censors conversions outside reporting windows, (iv) aggregation-threshold suppression removes sparse reporting cells, (v) randomized reporting noise perturbs released aggregates, and (vi) segment-heterogeneous signal loss makes these degradations vary across customer or campaign groups. These layers correspond to common primitives in modern reporting systems, including event-level attribution reports, summary reports, aggregation services, attribution windows, contribution budgets, and interoperable private-attribution protocols \citep{privacySandboxARA,privacySandboxPrivateAggregation,ipa2023interoperable,aksu2024summary,cookieMonster2024}. Figure~\ref{fig:intro-info} summarizes the resulting decision problem. The analyst observes the degraded report and must decide whether the evidence is strong enough to certify positive incrementality.

\begin{figure}[t]
  \centering
  \includegraphics[width=\linewidth]{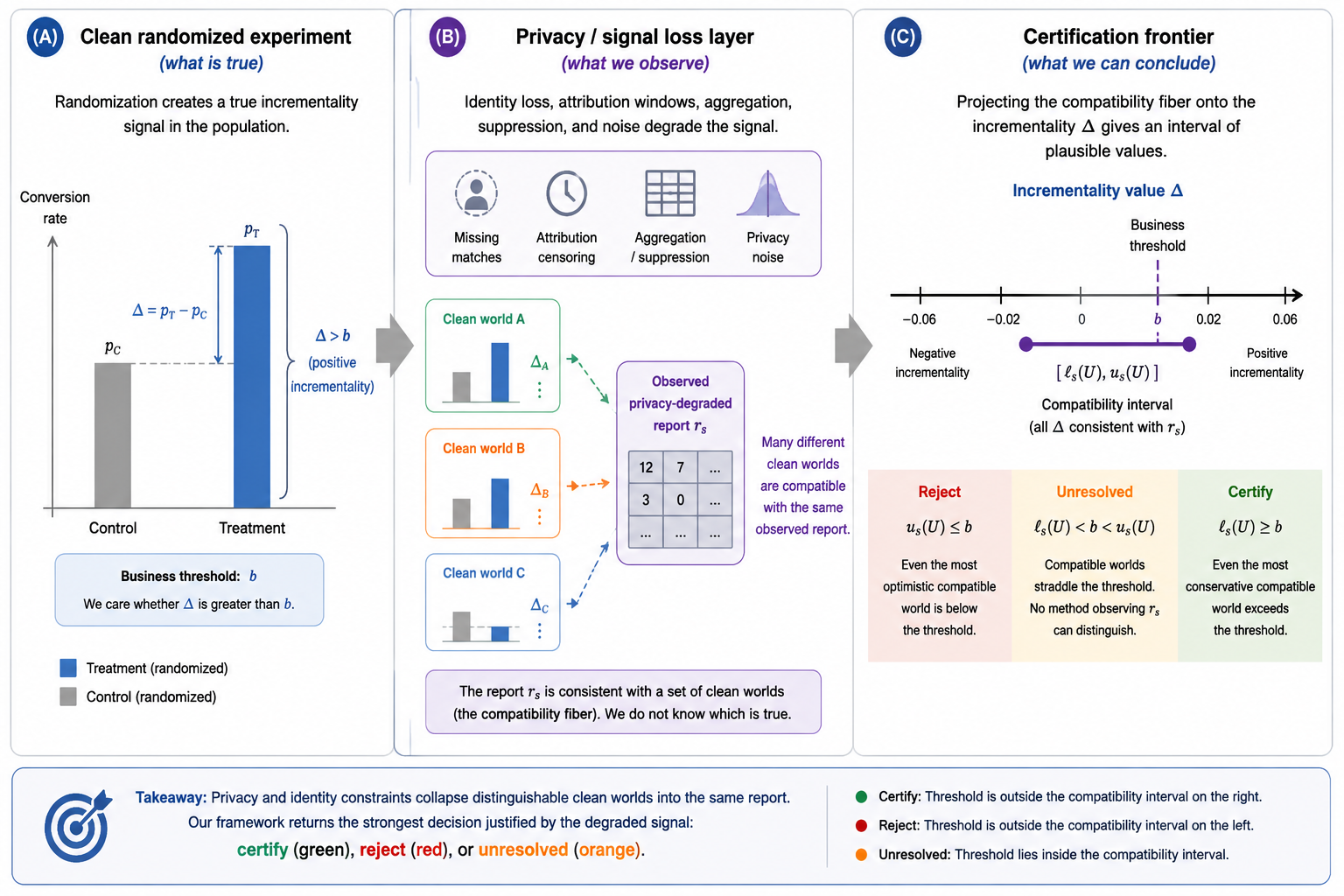}
  \caption{\small Privacy-robust incrementality measurement as a decision problem. A randomized advertising experiment identifies clean lift before reporting constraints are applied (as summarized in Table \ref{tab:privacy-api-map}). The proposed framework maps the resulting privacy-degraded report to a decision/certification frontier and returns the strongest justified decision: (i) certify positive incrementality, (ii) reject it, or (iii) mark the claim unresolved.}
  \label{fig:intro-info}
\end{figure}

The key distinction is between estimating lift and certifying a decision. A point estimate can be optimistic when high-value conversions are more likely to be unobserved, when low-volume segments are suppressed, or when randomized reporting noise dominates the signal. A deployable measurement system should therefore state what remains justified under plausible signal-loss mechanisms. We propose a framework that maps an observed privacy-degraded experiment into three sets. The first contains campaigns or segments whose lower-bound incremental value exceeds the business threshold. The second contains cases whose upper bound is below the threshold. The third contains unresolved cases where the available signal cannot support a stronger claim. 

In essence, the paper is motivated by four observations from advertising measurement practice. First, randomized incrementality data remain the strongest empirical anchor for causal claims, while missing measurement remains a separate obstacle. Second, privacy degradation changes which causal claims are identifiable, so it belongs inside the decision rule. Third, public datasets rarely include both the pre-privacy identity graph and post-privacy reports, so empirical work must be honest about what is actually observed. And fourth, business decisions require conservative certification, beyond ranking or attribution alone.

Overall, the research contributions are fivefold. \textbf{First}, we define a signal-loss model for privacy-constrained ads measurement that covers match-rate loss, linkability loss, attribution-window loss, aggregation-threshold suppression, randomized reporting noise, and segment-heterogeneous signal loss. \textbf{Second}, we introduce a compatibility-fiber view of privacy-degraded reports and specialize sharp partial-identification bands to privacy-constrained incrementality decisions over conversions and spend. \textbf{Third}, we show that the unresolved region is information-theoretic in the sense that if the business threshold lies inside the sharp band, no uniformly valid procedure observing the same degraded report can certify or reject the claim. \textbf{Fourth}, we prove sample-complexity and minimax lower-bound results that separate sampling error, reporting-noise error, and irreducible signal-loss width, with finite-sample certification handled as the direct simultaneous-coverage decision rule. And \textbf{fifth}, we define reporting-granularity diagnostics that trade off randomized reporting-noise reduction against segment-heterogeneous signal loss and make segment-level unresolved claims visible.

\section{Related Work}

\noindent \textbf{Incrementality, uplift modeling, and treatment targeting}:
Uplift modeling studies how to predict heterogeneous treatment effects and target treatment to units whose outcomes are most improved by treatment \citep{gutierrez2017causal,guelman2015ensemble,rudas2018linear,olaya2020survey}. The Criteo uplift benchmarks made large-scale randomized advertising data available for uplift modeling and individual-treatment-effect prediction \citep{diemert2018large,diemert2021largeBenchmark}. Recent work studies area under the uplift curve (AUUC) optimization and generalization guarantees for treatment targeting \citep{betlei2020treatment}. Our work asks a complementary measurement question: whether an observed randomized lift signal remains decision-certifying after privacy and identity degradation.

\noindent \textbf{Advertising measurement and attribution}:
Large field experiments have shown that observational and model-based advertising measurement can differ substantially from randomized lift estimates \citep{gordon2019comparison,gordon2023close,johnson2023inferno}. Related work studies ghost ads, geo experiments, causal impact, incremental return on ad spend, latent stratification, counterfactual demand-side platform (DSP) measurement, Facebook lift-study power calculations, and experiment-to-campaign prediction \citep{johnson2017ghost,lewis2015near,vaver2011geo,chen2019robust,brodersen2015inferring,berman2019latent,chan2017counterfactual,sun2015causal,liu2018designing,gordon2023predictive}. Attribution models and multi-touch attribution attempt to assign conversion credit across touchpoints, including causal and recurrent approaches \citep{diemert2017attribution,du2019causally,yao2022causalmta}. Our framework differs in its decision target. Attribution distributes credit, while privacy-robust incrementality asks what causal decision remains certified when reporting signals are degraded.

\noindent \textbf{Privacy-preserving measurement and signal loss}:
Differential privacy provides formal protection against individual-level disclosure \citep{dwork2008differential,dwork2014algorithmic}. Recent advertising-measurement work studies private conversion measurement, Attribution Reporting and Private Aggregation in Privacy Sandbox, Interoperable Private Attribution (IPA), summary-report optimization, on-device budgeting for differentially private ad measurement, aggregated event measurement, and privacy-preserving conversion prediction \citep{delaney2024differentially,privacySandboxARA,privacySandboxPrivateAggregation,ipa2023interoperable,aksu2024summary,ghazi2025sandbox,cookieMonster2024,metaAEM2021,clickWithoutCompromise2024,ibex2022,privacyPostClick2022}. These papers focus on privacy protocols, query release, and reporting accuracy. We focus on the causal decision layer downstream of such mechanisms. Given degraded or private reports from a randomized incrementality test, when is positive incrementality still certifiable?

\noindent \textbf{Partial identification, sensitivity, and robust decision rules}:
When data are incomplete or assumptions are weak, causal effects may be partially identified rather than point identified \citep{horowitz2000nonparametric,imbens2004confidence,manski2000monotone,richardson2014nonparametric,coppock2017double,gabriel2022nonparametric}. Recent work also studies causal inference with corrupted data and differentially private average treatment effect (ATE) or conditional average treatment effect (CATE) estimation \citep{agarwal2021corrupted,ohnishi2025locally,niu2022private}. Robust policy and off-policy evaluation methods emphasize decision rules that remain valid under confounding, weak support, or high-variance logged feedback \citep{kallus2018confounding,swaminathan2015counterfactual,su2020doubly}. Recent advertising-marketplace work applies a similar decision-object perspective to adjacent problems, including mechanism-robust online experiment design under interference \citep{shekhar2026interferenceDesign} and support-aware offline reserve-policy selection from logged auctions \citep{shekhar2026supportAwareReserve}. This paper builds on those ideas by developing a privacy-degraded incrementality \emph{decision frontier}, which is a specialization of sharp partial-identification logic to advertising reports whose signal-loss layers determine whether a lift claim is certifiable, rejectable, or impossible to resolve, together with finite-sample and information-theoretic limits for that frontier.

\section{Problem Setup}

We observe a randomized marketing experiment with treatment assignment
\(A_i\in\{0,1\}\), outcome \(Y_i\), and reported segment membership \(i\in s\).
The outcome may represent conversion, visit, spend, or advertiser value. Let
\(Y_i(1)\) and \(Y_i(0)\) denote potential outcomes. For a segment \(s\), the
clean incrementality estimand is
\begin{equation}
\label{eq:clean-ate}
\Delta_s
=
\E\!\left[Y_i(1)-Y_i(0)\mid i\in s\right].
\end{equation}
Randomization identifies this estimand under clean observation. The measurement problem is that the analyst may not observe the clean arm
means needed to estimate \(\Delta_s\). Instead, the reporting system releases a
privacy-degraded segment report \(r_s\). In the simplest mean-reporting case,
\(r_s\) contains degraded treated and control arm means
$
\widetilde\mu_{1,s},
$ $
\widetilde\mu_{0,s},
$
together with cell counts, suppression indicators, and any reporting-noise
metadata. The degraded arm mean is related to the clean arm mean by a
reportable-retention multiplier
\[
\widetilde\mu_{a,s}=q_{a,s}\mu_{a,s},
\qquad
\mu_{a,s}=\E[Y_i(a)\mid i\in s],
\qquad
a\in\{0,1\}.
\]
The multiplier \(q_{a,s}\in(0,1]\) summarizes the share of outcome signal that
remains matchable, attributable, and linkable to arm \(a\) in segment \(s\).
Later sections allow \(q_{a,s}\) to vary over an ambiguity set rather than
assuming it is exactly known.

Privacy-preserving reports may also suppress small cells or add randomized
reporting noise. A generic released cell total has the form
\[
\widetilde S_s
=
\mathbf 1\{N_s\ge m\}
\left(
N_s\widetilde\mu_s+\xi_s
\right),
\]
where \(N_s\) is the reported cell count, \(m\) is an aggregation threshold,
and \(\xi_s\) is randomized reporting noise. This notation is only meant to
make the reporting layer concrete. The rest of the paper works directly with
the released report \(r_s\), the degraded arm means it implies, and the set of
clean experimental worlds that remain compatible with that report.

The decision problem is to determine whether the clean incrementality value
\(\Delta_s\) remains certified after signal loss. Let \(b_s\ge0\) denote a
business threshold for meaningful incremental value. A valid measurement
procedure should certify segment \(s\) only when the degraded report supports
$
\Delta_s>b_s
$
uniformly over plausible signal-loss mechanisms.

\section{Signal-Loss Ambiguity Sets}

We describe privacy and identity degradation through an ambiguity set \(\U\). Each element \(u\in\U\) specifies a plausible signal-loss mechanism. A mechanism may include treatment-specific match-rate loss, segment-specific match-rate loss, linkability loss, attribution-window loss, aggregation-threshold suppression, randomized reporting noise, and segment-heterogeneous signal loss. We write
\begin{equation}
\label{loss}
    u=(\pi_{a,s},\kappa_s, \rho_{a,s}, m,\sigma_{\xi},\eta_s),
\end{equation}

where \(\pi_{a,s}\) is the match probability for treatment arm \(a\) in segment \(s\), \(\kappa_s\) controls linkability loss, \(\rho_{a,s}\) is attribution-window retention, \(m\) is the aggregation threshold, \(\sigma_{\xi}\) is the randomized reporting-noise scale, and \(\eta_s\) captures segment-heterogeneous signal loss.

The ambiguity set can be specified from product knowledge, data quality diagnostics, privacy-system documentation, or stress-test ranges. In a deployed ads setting, these ranges should be read from the reporting contract whenever possible. Attribution Reporting and IPA-style systems expose a small number of operational primitives such as  whether reports are event-level or aggregate, which attribution windows are active, which keys define summary reports, which small cells are suppressed or clipped, and what randomized reporting noise or contribution budget is applied. Table~\ref{tab:privacy-api-map} gives the mapping used in this paper.

\begin{table}[t]
  \centering
  \caption{\small Mapping the paper's signal-loss layers to deployed ads-reporting primitives and model components. Here $q_{a,s}$ is the reportable-retention multiplier defined in (\ref{q}). The empirical stress layers are simplified, auditable instantiations of these primitives over public randomized datasets, not proprietary production API traces.}
  \label{tab:privacy-api-map}
  \small
  \begin{tabular}{p{0.28\linewidth}p{0.31\linewidth}p{0.31\linewidth}}
    \toprule
    Signal-loss layer & Examples in modern systems & Model component \\
    \midrule
    Match-rate loss & Missing outcome matches, opt-outs, unmatched conversions & \(\pi_{a,s}\), \(q_{a,s}\) \\
    Linkability loss & Browser, device, or identity fragmentation across events & \(\kappa_s\), \(q_{a,s}\) \\
    Attribution-window loss & Event-level report windows, delayed or windowed conversion attribution & \(\rho_{a,s}\), \(q_{a,s}\) \\
    Aggregation-threshold suppression & Summary reports, conversion-key aggregates, minimum aggregation sizes & \(m\), unresolved cells \\
    Randomized reporting noise & Aggregation-service noise, private summary release, noisy aggregate measurement & \(\xi_g\), \(\sigma_\xi\) \\
    Segment-heterogeneous signal loss & Segment-specific reporting quality, trigger coarsening, prioritized events & \(\eta_s\), segment-specific \(\U_s\) \\
    \bottomrule
  \end{tabular}
\end{table}

For example,
\[
\pi_{a,s}\in[\underline \pi_s,\overline \pi_s],
\qquad
\rho_{a,s}\in[\underline \rho_s,\overline \rho_s],
\qquad
\sigma_\xi\in[0,\overline\sigma_\xi],
\]
with monotonicity or symmetry restrictions when justified. The analyst states the range of signal-loss mechanisms under which the decision claim should remain valid, and the framework certifies only claims that hold across that range.

This mapping is the link between the theoretical ambiguity set and real ads privacy systems. A Privacy Sandbox summary report, for instance, naturally determines a reporting partition, an aggregation threshold or contribution constraint, and a noise scale for \(\xi_g\). Event-level reporting determines attribution windows and coarsened trigger values, which enter through attribution-window loss and segment-heterogeneous signal loss. IPA-style protocols similarly return aggregate, purpose-limited measurements rather than row-level conversion labels, so their output is naturally represented as a released report \(r_s\) with an associated compatibility fiber. The paper's experiments use simplified, auditable stress layers as reproducible approximations to vendor-specific implementation traces. They test whether incrementality decisions survive the signal losses that these systems impose.

\begin{assumption}[Randomization and bounded outcomes]
\label{assump:randomization}
Treatment is randomized with known assignment probabilities bounded away from zero and one. Outcomes satisfy \(Y_i(a)\in[0,B]\) for \(a\in\{0,1\}\).
\end{assumption}

\begin{assumption}[Signal-loss ambiguity]
\label{assump:ambiguity}
The unknown measurement mechanism belongs to a known ambiguity set \(\U\). For each \(u\in\U\), the mechanism specifies loss parameters (defined in (\ref{loss})) that determine the distribution of the observed report conditional on the clean randomized experiment.
\end{assumption}

\noindent \textbf{Compatibility fibers}:
The information geometry of the problem is captured by the set of clean experimental worlds that remain compatible with the released report. Let \(r_s\) denote the privacy-degraded report for segment \(s\), including observed treated and control summaries, reporting thresholds, and randomized reporting-noise metadata. For a clean world \(\eta_s\), which contains the clean arm means and any latent match or attribution quantities needed to define the estimand, write \(T_u(\eta_s)\) for the report that would be released under signal-loss mechanism \(u\). The compatibility fiber is
\begin{equation}
\label{eq:compatibility-fiber}
\F_s(r_s;\U)
=
\{\eta_s:\text{ there exists }u\in\U\text{ such that }T_u(\eta_s)\text{ is compatible with }r_s\}.
\end{equation}
The word ``compatible'' means exact equality for deterministic degraded reports and membership in a high-probability confidence region when randomized reporting noise is present. The sharp population question is therefore geometric, asking to project the fiber \(\F_s(r_s;\U)\) onto the incrementality functional \(\Delta_s(\eta_s)\). If the projection lies entirely above the business threshold, the claim is certifiable. If it lies entirely below the threshold, the claim is rejectable. If it crosses the threshold, the released report has not preserved enough information to support either conclusion.

\section{Certified Incrementality Bounds}

\subsection{From compatibility fibers to business decisions}

Let \(\Delta_s(\eta_s)\) denote the clean incrementality value in segment \(s\) for a clean world \(\eta_s\). The sharp population bounds induced by a released report \(r_s\) are
\begin{equation}
\label{eq:fiber-projection}
\ell_s^\star(r_s;\U)
=
\inf_{\eta_s\in\F_s(r_s;\U)} \Delta_s(\eta_s),
\qquad
u_s^\star(r_s;\U)
=
\sup_{\eta_s\in\F_s(r_s;\U)} \Delta_s(\eta_s).
\end{equation}
These are identification limits caused by signal loss, with sampling uncertainty added later. Let \(\LB_s(\alpha;\U)\) and \(\UB_s(\alpha;\U)\) denote finite-sample lower and upper outer bounds for the clean incrementality value in segment \(s\), constructed so that they contain the sharp population fiber projection with probability at least \(1-\alpha\). Their concrete plug-in form is given in Eq.~\eqref{eq:finite-sample-band} later. The operational decision rule is
\begin{equation}
\label{eq:certification-rule}
\mathrm{Decision}(s)
=
\begin{cases}
\mathrm{Certify}, & \LB_s(\alpha;\U)>b_s,\\
\mathrm{Reject}, & \UB_s(\alpha;\U)\le b_s,\\
\mathrm{Unresolved}, & \text{otherwise.}
\end{cases}
\end{equation}
This rule deliberately treats unresolved cases as a valid output. If privacy degradation destroys too much information relative to the business threshold \(b_s\), the correct conclusion is a weaker decision claim.

This decision rule is the practitioner-facing version of the information geometry. The analyst asks what can be certified for all mechanisms in \(\U\), without first selecting a single true signal-loss mechanism. The sharpness of this logic is given by the certification-frontier in Theorem \ref{thm:sharp-frontier}, and the general fiber-projection result in Lemma \ref{lem:fiber-projection} in the appendix; the finite-sample validity of the rule is formalized afterward.

\subsection{Sharp match and attribution-loss frontier}

The fiber projection becomes closed form for the most common ads measurement failure mode. Suppose outcomes are nonnegative and bounded by \(B\). In segment \(s\), let the observed treated and control means be \(\widetilde \mu_{1s}\) and \(\widetilde \mu_{0s}\). Let \(q_{a,s}\) denote the combined reportable-retention multiplier in arm \(a\), collecting match-rate loss, attribution-window loss, and linkability loss when these channels are not separately identifiable. In this subsection \(q_{a,s}\) is an outcome-signal multiplier satisfying \(\widetilde\mu_{a,s}=q_{a,s}\mu_{a,s}\). It captures retained outcome signal after reporting loss, while a raw user match rate captures only one possible source of that retention. In the simplest match-attribution case,
\begin{equation}
\label{q}
q_{a,s}=\pi_{a,s}\rho_{a,s}\kappa_s\in[\underline q_{a,s},\overline q_{a,s}],
\qquad
0<\underline q_{a,s}\le \overline q_{a,s}\le 1.
\end{equation}
For nonnegative outcomes, the clean arm mean \(\mu_{a,s}=\E[Y_i(a)\mid i\in s]\) must satisfy
\[
\frac{\widetilde \mu_{a,s}}{\overline q_{a,s}}
\le
\mu_{a,s}
\le
\min\left\{
B,\frac{\widetilde \mu_{a,s}}{\underline q_{a,s}}
\right\}.
\]
The resulting population signal-loss band is
\begin{equation}
\label{eq:match-loss-band}
\ell_s(\U)
=
\frac{\widetilde \mu_{1s}}{\overline q_{1,s}}
-
\min\left\{B,\frac{\widetilde \mu_{0s}}{\underline q_{0,s}}\right\},
\qquad
u_s(\U)
=
\min\left\{B,\frac{\widetilde \mu_{1s}}{\underline q_{1,s}}\right\}
-
\frac{\widetilde \mu_{0s}}{\overline q_{0,s}}.
\end{equation}
The width of this band is a direct measure of privacy-induced information loss. It widens when retention rates are poorly known, when treatment and control retention can differ, or when the observed outcomes are sparse.

\begin{theorem}[Sharp privacy-loss certification frontier]
\label{thm:sharp-frontier}
Under the bounded nonnegative outcome model above, with positive retention lower bounds \(0<\underline q_{a,s}\le \overline q_{a,s}\), Eq.~\eqref{eq:match-loss-band} is sharp. For every \(\delta\in[\ell_s(\U),u_s(\U)]\), there exists a clean randomized experiment and a retention mechanism in \(\U\) that produce the same degraded arm means and satisfy \(\Delta_s=\delta\). Consequently:
\begin{enumerate}[leftmargin=*,itemsep=1pt]
  \item if \(\ell_s(\U)>b_s\), positive incrementality is certifiable uniformly over \(\U\);
  \item if \(u_s(\U)\le b_s\), non-incrementality relative to \(b_s\) is rejectable uniformly over \(\U\);
  \item if \(\ell_s(\U)\le b_s<u_s(\U)\), the released report is intrinsically unresolved. More precisely, there exist two clean worlds and two mechanisms in \(\U\) that induce the same released-report distribution, one with \(\Delta_s\le b_s\) and one with \(\Delta_s>b_s\). Any binary rule that must decide between certification and non-certification from the released report has worst-case error at least \(1/2\), and any uniformly valid three-state rule must leave this report unresolved.
\end{enumerate}
\end{theorem}

The theorem is the paper's main identification result. Existing missing-data bounds show that causal effects may be partially identified when outcomes are not fully observed. Here, partial-identification logic becomes a sharp measurement decision frontier for privacy-degraded ads reports. The frontier has both a constructive and an impossibility side. Outside the band, every compatible clean world supports the same decision. Inside the band, the same released report can be generated by clean worlds on opposite sides of the business threshold, so a forced binary decision has worst-case error at least \(1/2\). The unresolved region is therefore an information-theoretic part of the compatibility fiber, where the degraded report has not preserved enough information to certify or reject the claim.

\subsection{Finite-sample and reporting-noise certification}
\label{sec:finite-sample-certification}

Population sharpness is not enough for deployment because reports are finite and may contain randomized reporting noise. We keep the match-attribution relation from the previous subsection, \(\widetilde\mu_{a,s}=q_{a,s}\mu_{a,s}\) with \(q_{a,s}\in[\underline q_{a,s},\overline q_{a,s}]\), and replace the population degraded mean by a simultaneous confidence interval. Let
\[
[\underline{\widetilde\mu}_{a,s}(\alpha),
\overline{\widetilde\mu}_{a,s}(\alpha)]
\]
be simultaneous confidence bounds for \(\widetilde\mu_{a,s}\), including bounded-outcome randomization uncertainty and randomized reporting-noise uncertainty. Since smaller retention implies a larger compatible clean mean, the clean arm mean lies in the outer interval
\[
\mu_{a,s}
\in
\left[
\frac{\underline{\widetilde\mu}_{a,s}(\alpha)}{\overline q_{a,s}},
\,
\min\left\{
B,
\frac{\overline{\widetilde\mu}_{a,s}(\alpha)}{\underline q_{a,s}}
\right\}
\right].
\]
Taking the smallest compatible treated mean minus the largest compatible control mean gives the finite-sample lower bound, and taking the largest compatible treated mean minus the smallest compatible control mean gives the finite-sample upper bound:
\begin{equation}
\label{eq:finite-sample-band}
\LB_s(\alpha;\U)
=
\frac{\underline{\widetilde\mu}_{1,s}(\alpha)}{\overline q_{1,s}}
-
\min\left\{B,\frac{\overline{\widetilde\mu}_{0,s}(\alpha)}{\underline q_{0,s}}\right\},
\qquad
\UB_s(\alpha;\U)
=
\min\left\{B,\frac{\overline{\widetilde\mu}_{1,s}(\alpha)}{\underline q_{1,s}}\right\}
-
\frac{\underline{\widetilde\mu}_{0,s}(\alpha)}{\overline q_{0,s}}.
\end{equation}

The finite-sample guarantee follows directly from simultaneous coverage of the degraded arm means. Suppose Assumptions~\ref{assump:randomization} and~\ref{assump:ambiguity} hold, and the degraded mean intervals used in Eq.~\eqref{eq:finite-sample-band} have simultaneous coverage at least \(1-\alpha\) over the reported segment collection. Then, with probability at least \(1-\alpha\),
\[
\LB_s(\alpha;\U)\le \Delta_s \le \UB_s(\alpha;\U)
\qquad
\text{for all reported segments }s.
\]
Consequently, Eq.~\eqref{eq:certification-rule} controls false certification and false rejection simultaneously over all reported segments and all signal-loss mechanisms in \(\U\).

This is the sampling-theory layer. It says that a certified claim has survived three filters at once: bounded-outcome randomization uncertainty, randomized reporting noise, and worst-case signal loss in \(\U\). A rejected claim has failed the same robust test. An unresolved claim is one whose finite-sample, privacy-degraded compatibility interval still crosses the threshold.

\subsection{Sample complexity and privacy utility}
\label{sec:sample-complexity}

The finite-sample bound also gives a sample-complexity diagnostic. Let \(q_{\min}=\min_{a,s}\underline q_{a,s}\). If \(q_{\min}\) is small, every observed degraded mean must be rescaled more aggressively to recover a clean mean, so randomization error is amplified. If the retention intervals are wide, there is also an irreducible identification width that more samples cannot remove.

\begin{proposition}[Sample complexity under privacy signal loss]
\label{prop:sample-complexity}
Consider a finite reported segment collection \(\C\) and balanced per-arm sample size \(n_s\ge n\) in every segment. Suppose \(Y_i(a)\in[0,B]\), \(q_{a,s}\in[\underline q_{a,s},\overline q_{a,s}]\), \(q_{\min}=\min_{a,s}\underline q_{a,s}>0\), and independent mean-zero randomized reporting noise added to each arm total is sub-Gaussian with variance proxy \(\sigma_\xi^2\). With probability at least \(1-\alpha\), all segment effects satisfy
\begin{equation}
\label{eq:sample-complexity-error}
\Delta_s
\in
\left[
\ell_s(\U)-\Rad_n,\,
u_s(\U)+\Rad_n
\right],
\qquad
\Rad_n
=
\frac{2B}{q_{\min}}\sqrt{\frac{\log(4|\C|/\alpha)}{2n}}
+
\frac{2\sigma_\xi}{q_{\min}n}\sqrt{2\log(4|\C|/\alpha)}.
\end{equation}
Consequently, achieving finite-sample radius \(\Rad_n\le\varepsilon\) is guaranteed if
\begin{equation}
\label{eq:sample-complexity-n}
n
\gtrsim
\max\left\{
\frac{B^2}{q_{\min}^2\varepsilon^2}\log\frac{|\C|}{\alpha},
\frac{\sigma_\xi}{q_{\min}\varepsilon}\sqrt{\log\frac{|\C|}{\alpha}}
\right\}.
\end{equation}
The remaining width \(u_s(\U)-\ell_s(\U)\) is an identification cost of signal-loss uncertainty. It persists as \(n\) grows unless the ambiguity set \(\U\) itself tightens.
\end{proposition}

The rate in Eq.~\eqref{eq:sample-complexity-n} is the coverage penalty in concrete form. If only half of outcomes are matchable and attributable, the sampling-driven requirement scales by roughly \(4\). If only one fifth are reliably retained, it scales by roughly \(25\). Sample size addresses sampling error, while measurement diagnostics address the population band \([\ell_s(\U),u_s(\U)]\). More impressions shrink \(\Rad_n\), and better diagnostics tighten the ambiguity set \(\U\).

\begin{theorem}[Minimax lower bound for signal-loss measurement]
\label{thm:minimax-signal-loss}
Consider one segment with binary conversion value \(Y_i(a)\in\{0,B\}\) for a known scale \(B>0\). Let
$
p_a=\Pr\{Y_i(a)=B\},$ $ a\in\{0,1\},
$
denote the clean conversion probability in arm \(a\), so the clean incrementality target is \(\Delta=B(p_1-p_0)\). In each arm, the analyst observes only \(\widetilde Y_i(a)=R_i(a)Y_i(a)\), where \(R_i(a)\sim\mathrm{Bernoulli}(q)\) is independent retention and \(q\in(0,1]\) is known. With \(n\) observations per arm, every estimator \(\widehat\Delta\) based only on the degraded observations satisfies
\[
\sup_{p_0,p_1\in[1/4,3/4]}
\E_{p_0,p_1}\left[
\left|\widehat\Delta-B(p_1-p_0)\right|
\right]
\ge
c\,\frac{B}{\sqrt{qn}},
\]
for a universal constant \(c>0\), whenever \(qn\) is large enough. Consequently, any estimator that attains constant-probability error at most \(\varepsilon\) uniformly over this class requires
$
n\ge c'\frac{B^2}{q\varepsilon^2}
$
for a universal constant \(c'>0\). If \(q\) is not known and only \(q\in[\underline q,\overline q]\) is known, then even with infinite data the minimax absolute error is at least
$
\frac{B}{4}\left(1-\frac{\underline q}{\overline q}\right).
$
\end{theorem}

The theorem explains why the sample-complexity penalty reflects a real loss of information. Retention loss reduces the effective number of observed conversion trials, so the clean-data rate is no longer attainable from the degraded report. In addition, uncertainty about the retention mechanism creates an identification floor that no estimator can remove from the released report alone. The lower bound is stated for binary conversion outcomes, while Proposition~\ref{prop:sample-complexity} gives a conservative sufficient bound for bounded outcomes and noisy reports. Together they give the same planning message. More samples can reduce finite-sample uncertainty, but information omitted by the privacy-degraded report remains unavailable unless the ambiguity set itself is tightened.

Aggregation-threshold suppression and randomized reporting noise affect certification through the standard error of released cell means. If a reported cell \(g\) has treated and control counts \(N_{1g}\) and \(N_{0g}\), and the report adds independent mean-zero noise with variance \(\sigma_\xi^2\) to each arm total, a conservative variance proxy for the released treated-control difference is
\begin{equation}
\label{eq:privacy-variance}
\widehat V_g^{\mathrm{priv}}
=
\widehat V_g^{\mathrm{rand}}
+
\sigma_\xi^2
\left(
\frac{1}{N_{1g}^2}
+
\frac{1}{N_{0g}^2}
\right).
\end{equation}
Here \(\widehat V_g^{\mathrm{rand}}\) is the usual randomization variance proxy for the clean difference in means, and the second term is the reporting-noise penalty after converting noisy arm totals into noisy arm means. This penalty is small for large cells and severe for small cells, which is why privacy-preserving reports often support aggregate campaign conclusions before they support fine segment-level claims.

For a two-arm randomized experiment with assignment-unit variance proxy \(\sigma_s^2\), effective per-arm cell size \(N_s\), and randomized reporting-noise variance \(\sigma_\xi^2\), the corresponding conservative minimum detectable effect is
\begin{equation}
\label{eq:privacy-mde}
\MDE_s^{\mathrm{priv}}
=
\left(z_{1-\alpha/2}+z_{1-\beta}\right)
\sqrt{
\frac{2\sigma_s^2}{N_s}
+
\sigma_\xi^2
\left(
\frac{1}{N_{1s}^2}
+
\frac{1}{N_{0s}^2}
\right)
}.
\end{equation}
The first quantile controls the two-sided type-I error rate \(\alpha\), and the second controls power \(1-\beta\). If the positive population signal-loss margin \(\ell_s(\U)-b_s\) is smaller than \(\MDE_s^{\mathrm{priv}}\), then certification is not expected to be statistically powered at level \((\alpha,\beta)\). The MDE formula is therefore a planning diagnostic. It separates cases where the campaign effect is too small from cases where randomized reporting noise, sparse cells, or retention uncertainty make the privacy-degraded report too weak for certification. Appendix Proposition~\ref{prop:privacy-mde} gives the derivation.

\subsection{Reporting granularity and segment geometry}
\label{sec:granularity}

Privacy reports are often released over a partition \(\calP\) of users, campaigns, geographies, devices, or time windows. Finer partitions are more interpretable and better aligned with segment safety, while smaller cells increase aggregation-threshold suppression. Coarser partitions reduce randomized reporting noise and missing-cell risk, while larger cells can hide segment-heterogeneous signal loss. This tradeoff can be made explicit.

To see the tradeoff mathematically, suppose the segment-level treatment effect function is \(H\)-Lipschitz over a segment metric \(d_{\mathrm{seg}}\). For any reported cell \(g\in\calP\), define
\[
\diam(g)=\sup_{s,s'\in g}d_{\mathrm{seg}}(s,s'),
\qquad
\bar\Delta_g=\sum_{s\in g}w_{s|g}\Delta_s,
\]
where \(\bar\Delta_g\) is the within-cell weighted average effect. The Lipschitz condition implies that every segment \(s\in g\) is no more than \(H\diam(g)\) below the cell average. Therefore, if the cell-level robust lower bound satisfies \(\bar\Delta_g\ge \LB_g(\alpha;\U)\) on the simultaneous coverage event, then every segment \(s\in g\) satisfies
\[
\Delta_s
\ge
\LB_g(\alpha;\U)-H\,\diam(g).
\]
Thus cell \(g\) supports segment-safe certification whenever
\begin{equation}
\label{eq:segment-safe}
\LB_g(\alpha;\U)-H\,\diam(g)>b_g.
\end{equation}

This calculation gives a practical rule for choosing reporting granularity. A coarse cell has larger \(N_{ag}\), so Eq.~\eqref{eq:privacy-variance} is smaller and aggregation-threshold suppression is less likely. The same coarse cell can have larger \(\diam(g)\), which increases the heterogeneity penalty \(H\diam(g)\). The empirical implementation estimates this penalty from the data by nesting the finest reporting cells inside each coarser reported cell, computing the treatment-effect dispersion of those finest cells, and using the maximum deviation from the coarser cell's weighted average as the observed \(H\diam(g)\) proxy. The best reporting partition is determined by the balance between randomized reporting-noise reduction and the segment heterogeneity it averages over.

\subsection{Segment safety under heterogeneous signal loss}

Aggregate incrementality can be certified even when subgroup claims are not. Let \(\C\) be a collection of segments, such as device types, channels, customer history groups, geography, or advertiser categories. The framework certifies segment safety by requiring
\[
\LB_s(\alpha/|\C|;\U)>b_s
\qquad
\text{for every }s\in\C
\]
or by reporting which segments remain unresolved after multiplicity correction. This step is important because signal loss is often heterogeneous. For example, match-rate loss may differ by channel, attribution-window loss may differ by customer history, and aggregation-threshold suppression may disproportionately affect small segments.

This segment-safety check is substantive. Appendix~\ref{ex:heterogeneous-signal-loss} gives a constructive example in which two signal-loss mechanisms induce the same aggregate degraded lift and lead the segment-level rule to certify different segments. The practical warning is that a campaign can appear incrementally positive in aggregate while the evidence is insufficient for one or more segments. The proposed decision object makes this visible by returning unresolved segment claims directly.

\section{Robust Incrementality Certification Algorithm}

Algorithm~\ref{alg:certification} summarizes the proposed method. The algorithm takes randomized experimental data, an ambiguity set of signal-loss mechanisms, reporting thresholds, randomized reporting-noise parameters, and business decision thresholds. It returns three sets indicating certified, rejected, and unresolved cases. In the match-attribution implementation used in the experiments, the worst case over \(\U\) has the closed form in Eq.~\eqref{eq:match-loss-band}, so match-rate loss, attribution-window loss, and linkability loss enter through the retained-signal interval \([\underline q_{a,s},\overline q_{a,s}]\).

\begin{algorithm}[t]
\caption{Privacy-robust incrementality certification}
\label{alg:certification}
\begin{algorithmic}[1]
\Require Randomized experiment \(\D\), segment collection \(\C\), ambiguity set \(\U\), business thresholds \(\{b_s\}_{s\in\C}\), error level \(\alpha\).
\Ensure Certified set \(\C_{\mathrm{cert}}\), rejected set \(\C_{\mathrm{reject}}\), and unresolved set \(\C_{\mathrm{unres}}\).
\State Identify the cell collection \(\C_{\mathrm{cell}}\subseteq\C\) induced by the reporting partition, and set the per-arm error level to \(\alpha_{\mathrm{arm}}=\alpha/\max\{2|\C_{\mathrm{cell}}|,1\}\).
\State Initialize the suppressed set \(\C_{\mathrm{supp}}\gets\emptyset\).
\For{each segment \(s\in\C\)}
  \State Estimate treated and control degraded means and arm counts in segment \(s\), the inputs to Eqs.~\eqref{eq:match-loss-band} and~\eqref{eq:finite-sample-band}.
  \If{the reporting count is below the aggregation threshold}
    \State Mark \(s\) as aggregation-suppressed by updating \(\C_{\mathrm{supp}}\gets\C_{\mathrm{supp}}\cup\{s\}\).
  \Else
    \State Convert \(\U\) to retained-signal endpoints \(\underline q_{a,s}\) and \(\overline q_{a,s}\) as in Eq.~\eqref{eq:match-loss-band}.
    \State Construct simultaneous degraded-mean bounds using \(\alpha_{\mathrm{arm}}\), the bounded-outcome Hoeffding radius, and the reporting-noise radius summarized in Eq.~\eqref{eq:sample-complexity-error}.
    \State Construct \(\LB_s(\alpha;\U)\) and \(\UB_s(\alpha;\U)\) by plugging the endpoint worst cases into Eq.~\eqref{eq:finite-sample-band}.
  \EndIf
\EndFor
\State \(\C_{\mathrm{cert}}\gets\{s\in\C\setminus\C_{\mathrm{supp}}:\LB_s(\alpha;\U)>b_s\}\) as in Eq.~\eqref{eq:certification-rule}.
\State \(\C_{\mathrm{reject}}\gets\{s\in\C\setminus\C_{\mathrm{supp}}:\UB_s(\alpha;\U)\le b_s\}\) as in Eq.~\eqref{eq:certification-rule}.
\State \(\C_{\mathrm{unres}}\gets\C\setminus(\C_{\mathrm{cert}}\cup\C_{\mathrm{reject}})\) as in Eq.~\eqref{eq:certification-rule}.
\State \Return \(\C_{\mathrm{cert}},\C_{\mathrm{reject}},\C_{\mathrm{unres}}\).
\end{algorithmic}
\end{algorithm}

The validity of Algorithm~\ref{alg:certification} is an immediate consequence of the finite-sample certification logic in Section~\ref{sec:finite-sample-certification}. On the simultaneous coverage event for the finite-sample bounds, the algorithm certifies only segments with \(\Delta_s>b_s\), rejects only segments with \(\Delta_s\le b_s\), and leaves the remaining segments unresolved by construction. The operational interpretation is simple. If the ambiguity set is wide, the algorithm returns more unresolved decisions. If the measurement system improves, the ambiguity set tightens and certification becomes easier.

\section{Empirical Results}

\subsection{Datasets}

The empirical analysis uses two randomized datasets.

\noindent \textbf{Criteo Uplift}:
The Criteo dataset contains anonymized covariates, treatment assignment, exposure, visits, and conversions from a large advertising incrementality setting \citep{diemert2018large,criteoUpliftDataset}. The experiments use \(2{,}000{,}000\) rows, with treated share \(0.850\) and conversion rate \(0.00293\). It is the main large-scale ads dataset. Its treatment imbalance and low conversion rates make it useful for studying sparse conversion certification and randomized reporting-noise sensitivity.

\noindent \textbf{Hillstrom email experiment}:
The Hillstrom dataset contains randomized email treatment arms, customer history features, visits, conversions, and spend \citep{hillstrom2008minethatdata}. The experiments use all \(64{,}000\) rows, with treated share \(0.667\), conversion rate \(0.00903\), and mean spend \(1.051\). It is smaller and more interpretable than Criteo. It is useful for segment-level decision analysis because customer history, channel, and demographic-like variables define natural reporting cells.

\subsection{Signal-loss stress layers}

Public datasets provide randomized treatment and outcomes without paired pre-privacy identity graphs and post-privacy production reports, so we introduce controlled signal-loss layers over real randomized experiments. Table~\ref{tab:privacy-api-map} gives the reporting-system interpretation of these layers; this subsection describes the corresponding implementation. The experimental implementation uses an explicit stress tuple consisting of match rate, attribution retention, linkable-identity retention, aggregation threshold, and reporting-noise scale. Match-rate loss is implemented by multiplying clean treated and control arm means by a match component that lowers the share of outcomes remaining matchable to the randomized experiment; attribution-window loss is implemented by a second retention component that lowers the matched share retained inside the reporting window; linkability loss is implemented by a third retention component that lowers the share of outcomes still linkable to the same experimental unit. In the aggregate certification frontier, these three components are set equal to \(q^{1/3}\), so their product is the reportable-retention parameter \(q_{\mathrm{true}}\in\{0.95,0.80,0.65,0.50,0.35,0.25\}\). The corresponding retention half-widths are \(\omega\in\{0.04,0.06,0.08,0.10,0.12,0.14\}\), so the decision rule constructs the interval \([q_{\mathrm{true}}-\omega,q_{\mathrm{true}}+\omega]\), clipped to \([0.02,1]\), and applies the sharp match-loss band and finite-sample band to the degraded arm means. Aggregation-threshold suppression is implemented by checking reported cell counts against \(m\in\{50,100,150,250,400,600\}\) for the six aggregate stress settings; suppressed cells are assigned the unresolved state rather than a numerical certificate, and in the granularity analysis the fixed reporting-cell threshold is \(m=100\). Randomized reporting noise is implemented as Gaussian arm-total uncertainty with scale \(\sigma_\xi\), which enters the finite-sample radius after division by the treated and control cell counts. The six aggregate stress settings use \(\sigma_\xi=(1.0,1.6,2.0,2.5,3.0,4.0)\) for Criteo and \(\sigma_\xi=(0.4,0.64,0.8,1.0,1.2,1.6)\) for Hillstrom. Segment-heterogeneous signal loss is separately implemented in the granularity diagnostic by letting each reporting cell have its own retention interval and noise scale as a function of cell size and a deterministic cell-level jitter, and by adding a data-driven heterogeneity penalty equal to the largest nested fine-cell treatment-effect deviation inside the reported cell. This keeps the empirical procedure aligned with the theory in the sense that the implementation uses auditable reporting-primitive stress tests over public randomized data, while production identity graphs and conversion timestamps remain unavailable.

\subsection{Population and finite-sample certification}

The first empirical distinction mirrors the two layers of the theory. The population signal-loss band asks what would be certifiable with exactly known degraded arm means and a retention mechanism known only through \(\U\). This is the object in Theorem~\ref{thm:sharp-frontier}. The finite-sample band then adds bounded-outcome randomization uncertainty, randomized reporting-noise uncertainty, and multiplicity correction across the reported cells within the evaluated report, as described in Section~\ref{sec:finite-sample-certification}. The six stress settings are evaluated as separate planning scenarios, with multiplicity correction applied within each report. A claim can therefore be population-certifiable and still finite-sample unresolved.

\begin{figure}[t]
  \centering
  \includegraphics[width=\linewidth]{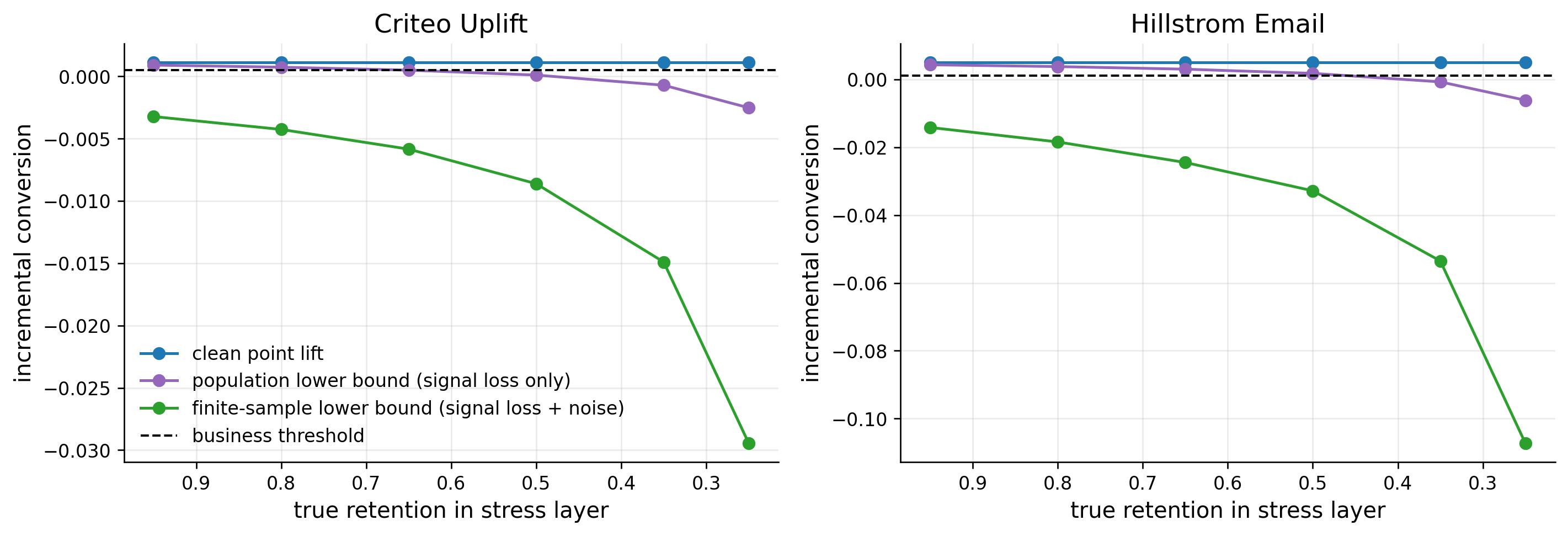}
  \caption{\small Population and finite-sample certification frontiers under auditable signal-loss stress layers. The purple curve is the population lower bound, which applies signal-loss uncertainty while treating the degraded arm means as exactly known. The green curve is the finite-sample lower bound, which adds simultaneous randomization uncertainty and randomized reporting-noise uncertainty to the same signal-loss band. A claim can therefore be population-certifiable and still finite-sample unresolved.}
  \label{fig:privacy-frontier}
\end{figure}

Figure~\ref{fig:privacy-frontier} shows the distinction directly. The population curve answers an \textbf{identification} question that if the degraded report were known without sampling error, would the signal-loss ambiguity set still allow certification? The finite-sample curve answers the \textbf{deployment} question that after adding simultaneous confidence radii and reporting-noise uncertainty, does the lower bound still clear the business threshold? In Criteo, the clean conversion lift is \(0.001117\), compared with a business threshold of \(0.0005\). The population robust lower bound certifies the claim at \(q=0.95\) and \(q=0.80\), with lower bounds \(0.000905\) and \(0.000740\). At \(q=0.65\), it reaches \(0.000499\), just below the threshold. The finite-sample lower bound is negative throughout, ranging from \(-0.00323\) at \(q=0.95\) to \(-0.02945\) at \(q=0.25\). In Hillstrom, the clean conversion lift is \(0.004955\), compared with a threshold of \(0.001\). Population certification survives through \(q=0.50\), where the lower bound is \(0.001743\), and becomes unresolved at \(q=0.35\) and below. As in Criteo, all finite-sample decisions are unresolved. This is the intended output of a conservative measurement system when the released report has limited information for a stronger deployed decision.

\subsection{Sample complexity and information loss}

The second diagnostic connects the finite-sample radius in Eq.~\eqref{eq:sample-complexity-error}, the sufficient sample-size condition in Eq.~\eqref{eq:sample-complexity-n}, and the lower-bound rate in Theorem~\ref{thm:minimax-signal-loss}. The left panel of Figure~\ref{fig:sample-complexity} evaluates \(\Rad_n\) from Eq.~\eqref{eq:sample-complexity-error} over a grid of per-arm sample sizes \(n\) and reportable-retention levels \(q_{\min}\), using \(B=1\), \(|\C|=24\), and \(\alpha=0.05\). The plotted curves set \(\sigma_\xi=0\) to isolate the retention-driven sampling penalty; the reporting-noise term in Eq.~\eqref{eq:sample-complexity-error} is included in the implementation and affects finite-sample certification in Figure~\ref{fig:privacy-frontier}. The right panel plots the constant-scaled minimax curve \(0.08\,B/\sqrt{qn}\), which is the implemented visualization of the universal-constant lower-bound rate in Theorem~\ref{thm:minimax-signal-loss}. Thus the figure compares the upper-bound planning radius used by the certification rule with the information-theoretic lower-bound scale that no estimator can uniformly beat.

\begin{figure}[t]
  \centering
  \includegraphics[width=\linewidth]{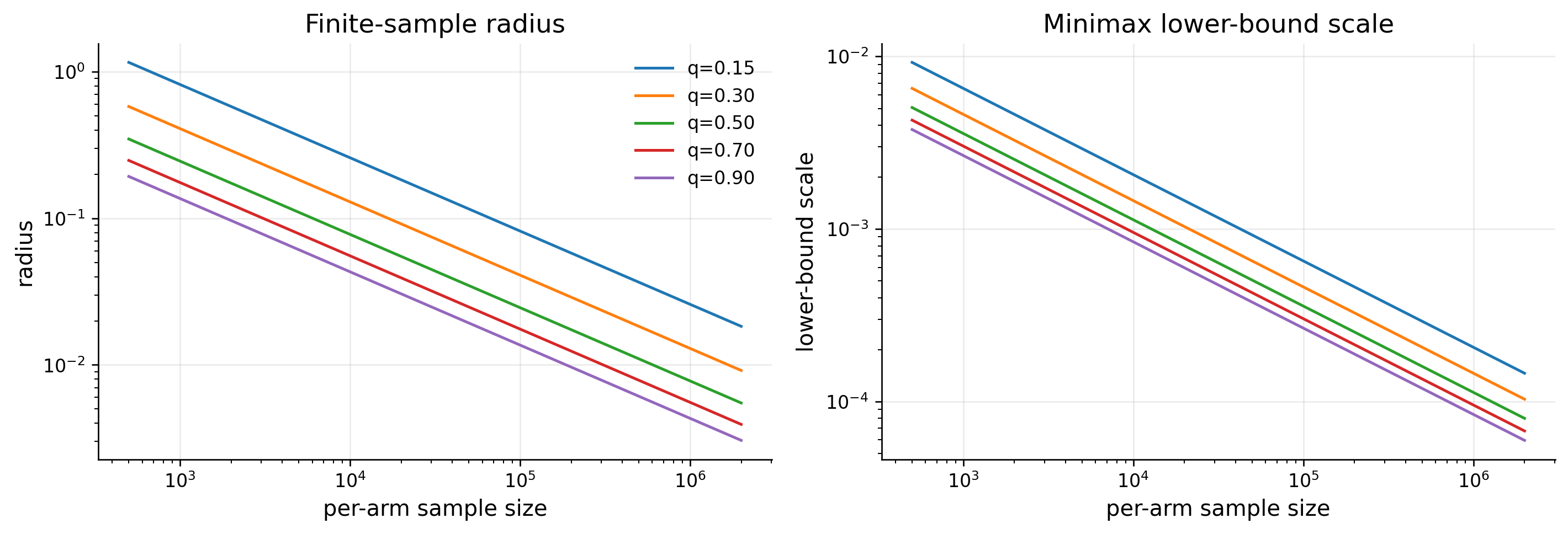}
  \caption{\small Sample-complexity and minimax scaling under privacy signal loss. The left panel evaluates the finite-sample radius \(\Rad_n\) from Eq.~\eqref{eq:sample-complexity-error} as the per-arm sample size and reportable-retention lower bound \(q_{\min}\) vary. The right panel evaluates the implemented constant-scaled minimax curve \(0.08\,B/\sqrt{qn}\), corresponding to the universal-constant lower-bound rate in Theorem~\ref{thm:minimax-signal-loss}. Lower retention enlarges both quantities, showing that privacy signal loss reduces effective information rather than merely adding a removable constant error.}
  \label{fig:sample-complexity}
\end{figure}

Figure~\ref{fig:sample-complexity} shows the expected \(1/q_{\min}^2\)-type sample burden from the first term of Eq.~\eqref{eq:sample-complexity-n}. To reach \(\Rad_n\le\varepsilon=0.02\), the sufficient per-arm sample size is \(186{,}669\) at \(q_{\min}=0.90\), \(604{,}807\) at \(q_{\min}=0.50\), and \(6.72\) million at \(q_{\min}=0.15\). For a tighter radius \(\varepsilon=0.005\), the corresponding \(q_{\min}=0.15\) requirement is \(107.5\) million per arm. The minimax curve gives the complementary impossibility statement from Theorem~\ref{thm:minimax-signal-loss} that when only a small share \(q\) of outcomes remains reportable, the effective information scale is \(qn\), so no estimator observing the same degraded report can uniformly recover clean incrementality at the clean-data rate. The sample-complexity panel and the minimax panel therefore make the same point from opposite directions. Eq.~\eqref{eq:sample-complexity-n} states how many observations are sufficient for the deployed bound to shrink, while Theorem~\ref{thm:minimax-signal-loss} shows why the \(q\)-dependent slowdown is unavoidable.

\subsection{Granularity and segment safety}

The third diagnostic tests the reporting-granularity logic in Section~\ref{sec:granularity} and the segment-safety logic above. The experiment forms increasingly fine reporting partitions in both datasets using binary conversion outcomes, \(\alpha=0.05\), bounded outcome scale \(B=1\), and business thresholds \(0.0005\) for Criteo and \(0.001\) for Hillstrom. The Criteo partitions are \(\{\texttt{f0}\}\), \(\{\texttt{f0},\texttt{f1}\}\), \(\{\texttt{f0},\texttt{f1},\texttt{f2}\}\), and \(\{\texttt{f0},\texttt{f1},\texttt{f2},\texttt{logged exposure}\}\), where the feature buckets are quartile buckets of the anonymized features. The Hillstrom partitions are \(\{\texttt{channel}\}\), \(\{\texttt{channel},\texttt{zip}\}\), \(\{\texttt{channel},\texttt{zip},\texttt{history}\}\), and \(\{\texttt{channel},\texttt{zip},\texttt{history},\texttt{recency}\}\), with recency also quartile-bucketed. For each reported cell \(g\) with row count \(n_g\), the implementation sets
\[
\widehat q_g
=
\mathrm{clip}\!\left(
0.70+
\mathrm{clip}\!\left(\frac{\log(n_g/500)}{20},-0.10,0.10\right)
+\zeta_g,\,
0.35,0.92
\right),
\]
where \(\zeta_g\in[-0.05,0.05]\) is a deterministic cell-level jitter used only to create reproducible segment-heterogeneous signal loss. The retention interval is
\[
\left[
\max\{0.02,\widehat q_g-\omega_g\},
\min\{1,\widehat q_g+\omega_g\}
\right],
\qquad
\omega_g=0.08+\frac{(250-n_g)_+}{2500},
\]
and the reporting-noise scale is
\[
\sigma_{\xi,g}=1+\frac{(200-n_g)_+}{200}.
\]
Cells with \(n_g<100\) are marked aggregation-suppressed and assigned the unresolved state. Non-suppressed cells are evaluated by the same population and finite-sample certification rules used in Figure~\ref{fig:privacy-frontier}, with the finite-sample correction applied simultaneously over the cells in the current report.

The segment-safety penalty is computed by nesting the finest partition inside every coarser reported cell. For each valid finest cell \(f\subset g\) with at least \(80\) rows and both treatment arms present, the implementation computes \(\widehat\Delta_f=\bar Y_{1f}-\bar Y_{0f}\). It then forms the row-weighted average \(\widehat{\bar\Delta}_g=\sum_{f\subset g}w_{f|g}\widehat\Delta_f\) and uses
\[
\widehat H\diam(g)
=
\max_{f\subset g}
\left|
\widehat\Delta_f-\widehat{\bar\Delta}_g
\right|
\]
as the empirical proxy for the heterogeneity term in Eq.~\eqref{eq:segment-safe}. The reported segment-safe lower bound is therefore the finite-sample lower bound minus this penalty.

\begin{figure}[t]
  \centering
  \includegraphics[width=\linewidth]{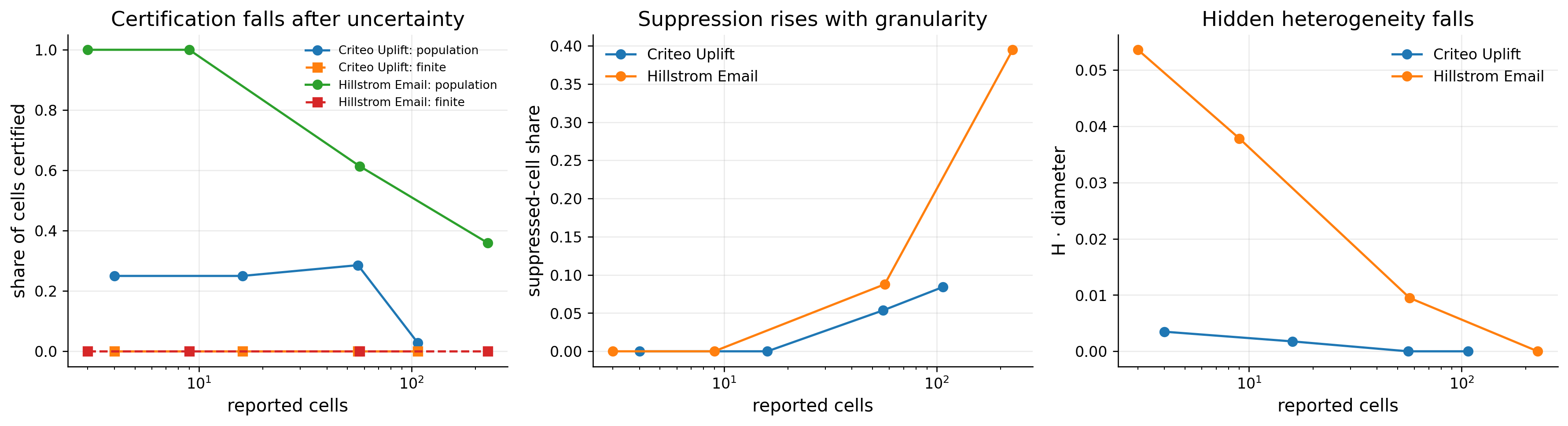}
  \caption{\small Reporting granularity diagnostics. The horizontal axis is the number of reported cells in the partition. The left panel compares the share of cells that clear the population signal-loss lower bound with the share that clear the finite-sample lower bound after simultaneous uncertainty and reporting noise are added. The middle panel reports aggregation-threshold suppression with fixed threshold \(m=100\). The right panel reports the empirical segment-heterogeneity penalty, computed as the largest nested fine-cell treatment-effect deviation inside each reported cell and used as a proxy for \(H\diam(g)\) in Eq.~\eqref{eq:segment-safe}. Finer reports reduce this heterogeneity penalty but create smaller cells and more suppression, so aggregate or population-level lift should not be read as subgroup safety.}
  \label{fig:granularity}
\end{figure}

Figure~\ref{fig:granularity} shows why segment certification is difficult under privacy constraints. In Criteo, population cell certification is about \(25\%\) for the coarse and medium partitions, \(28.6\%\) for the fine partition, and only \(2.8\%\) for the very fine partition. In Hillstrom, population certification is \(100\%\) for coarse and medium partitions, \(61.4\%\) for the fine partition, and \(36.0\%\) for the very fine partition. Finer reports reduce the empirically estimated nested-cell heterogeneity penalty and create sparse cells with aggregation-threshold suppression. Aggregation-threshold suppression reaches \(8.4\%\) in the finest Criteo partition and \(39.5\%\) in the finest Hillstrom partition. In the current stress layers, every finite-sample cell-level decision is unresolved and every segment-safe decision is also unresolved under the three-state rule. This agrees with the theory in the sense that the framework should refuse a subgroup claim when the lower bound, randomized reporting noise, multiplicity correction, and heterogeneity penalty do not jointly clear the business threshold.

\section{Discussion}

The proposed framework treats privacy-preserving ads measurement as a causal decision problem under incomplete observation. This framing differs from both pure attribution and pure privacy-protocol design. Attribution asks how to allocate credit among touchpoints. Privacy protocols ask how to report measurements while protecting individuals. Privacy-robust incrementality asks what causal claim remains justified after the measurement signal has been degraded.

This distinction matters for platform decisions. A privacy-preserving report can be technically valid yet decision-insufficient. A randomized experiment can be causally valid in design while linkability loss, attribution-window loss, or randomized reporting noise still prevent fine-grained lift certification. The framework makes this insufficiency explicit by outputting unresolved claims instead of overconfident estimates.

\section{Limitations and Future Work}

The framework depends on the ambiguity set \(\U\). If the analyst excludes plausible signal-loss mechanisms, the certification rule may be too optimistic. The empirical design also uses controlled stress layers because public datasets lack the complete pre-privacy and post-privacy measurement process. The stress layers are semantics-faithful abstractions of modern ads reporting primitives, and their numerical bands should be interpreted as reproducible stress-test results for the stated public-data setting. This limitation reflects the realistic setting for reproducible research on privacy-constrained advertising measurement. Future work should evaluate the framework against production reporting logs or browser-level API traces, richer identity graphs, auction-level ads data, and adaptive experimentation policies that choose between additional measurement and launch decisions. Appendix~\ref{app:additional-diagnostics} provides additional diagnostics that isolate the fiber geometry, privacy-aware MDE behavior, and heterogeneous signal-loss reversals.

\section{Conclusion}

Privacy constraints are now part of the statistical environment for advertising measurement. Randomization remains essential, and privacy-degraded reports require an additional certification layer when match-rate loss, linkability loss, attribution-window loss, aggregation-threshold suppression, randomized reporting noise, and segment-heterogeneous signal loss are present. This paper proposes a privacy-robust incrementality framework that converts randomized marketing evidence into certified, rejected, and unresolved decision claims. The empirical results reinforce the theoretical message. Positive clean lift and favorable population lower bounds can still fall short of deployment certification once finite-sample uncertainty and randomized reporting noise are included. The contribution is a frontier-based certification layer for modern ads systems. It preserves the strength of randomized experiments while making explicit when privacy-degraded data are insufficient for a causal decision.

\paragraph{Code availability.}
The implementation is available in the \href{https://github.com/p-shekhar/privacy-incrementality-measurement.git}{GitHub repository}.

\bibliographystyle{unsrtnat}
\bibliography{references}

\newpage
\appendix

\section{Additional Empirical Diagnostics}
\label{app:additional-diagnostics}

The main paper reports three empirical diagnostics: the certification frontier, the sample-complexity and minimax scaling curves, and the reporting-granularity tradeoff. The appendix records the supporting diagnostics used to validate the remaining theoretical components.

\subsection{Compatibility fiber diagnostics}

Figure~\ref{fig:fiber-appendix} isolates the information-geometric object behind Theorem~\ref{thm:sharp-frontier} and Lemma~\ref{lem:fiber-projection}. The degraded report has treated mean \(0.12\), control mean \(0.08\), and retention interval \([0.5,1.0]\). The corresponding sharp clean incrementality fiber is \([-0.04,0.16]\). A business threshold inside this interval is intrinsically unresolved from the released report, while thresholds outside the interval are decidable at the population level. This diagnostic is deliberately one-dimensional because it makes the core geometry visible.

\begin{figure}[t]
  \centering
  \includegraphics[width=0.86\linewidth]{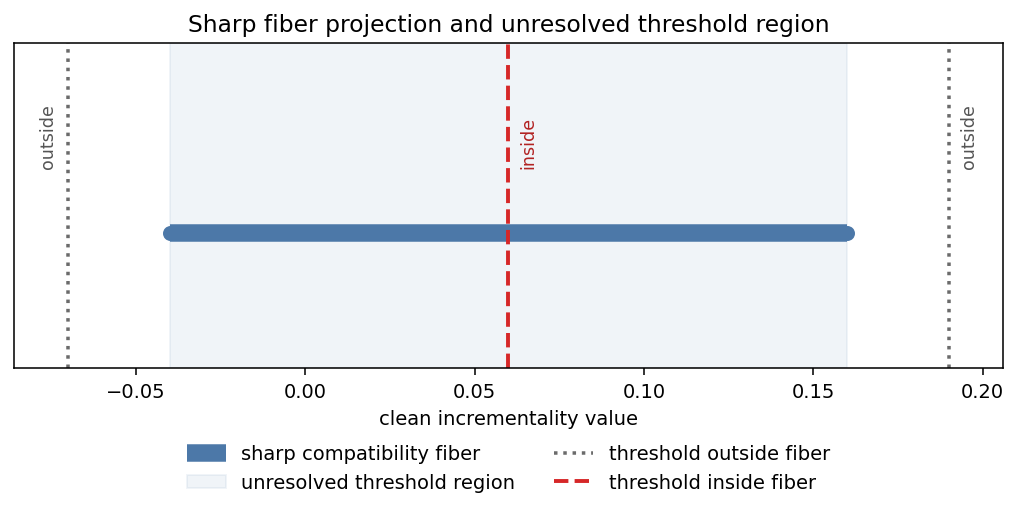}
  \caption{\small Sharp compatibility fiber for a controlled degraded-report example. This diagnostic is not dataset-specific. It fixes degraded treated and control means \(\widetilde\mu_1=0.12\) and \(\widetilde\mu_0=0.08\), with retention ambiguity \(q\in[0.5,1.0]\). The shaded interval is the full set of clean incrementality values compatible with the same released report and ambiguity set. A business threshold below the interval is certifiable, a threshold at or above the interval is rejectable, and a threshold inside the interval is unresolved because compatible clean worlds exist on both sides of the decision boundary.}
  \label{fig:fiber-appendix}
\end{figure}

\subsection{Privacy-aware detectable effects}

Figure~\ref{fig:mde-appendix} supports Proposition~\ref{prop:privacy-mde}. It evaluates Eq.~\eqref{eq:privacy-mde} over effective per-arm cell sizes from \(100\) to \(300{,}000\) and reporting-noise scales \(\sigma_\xi\in\{0,1,3,10\}\), using \(\sigma_s=0.5\), \(\alpha=0.05\), \(\beta=0.2\), and \(N_{1s}=N_{0s}=N_s\). At \(N_s=100\), the privacy-aware MDE is \(0.198\) without reporting noise and \(0.443\) when \(\sigma_\xi=10\). At \(N_s=300{,}000\), all curves are near \(0.00362\). Thus randomized reporting noise is most damaging for small cells, which is the same regime where aggregation-threshold suppression and segment-level reporting are operationally most tempting.

\begin{figure}[t]
  \centering
  \includegraphics[width=0.70\linewidth]{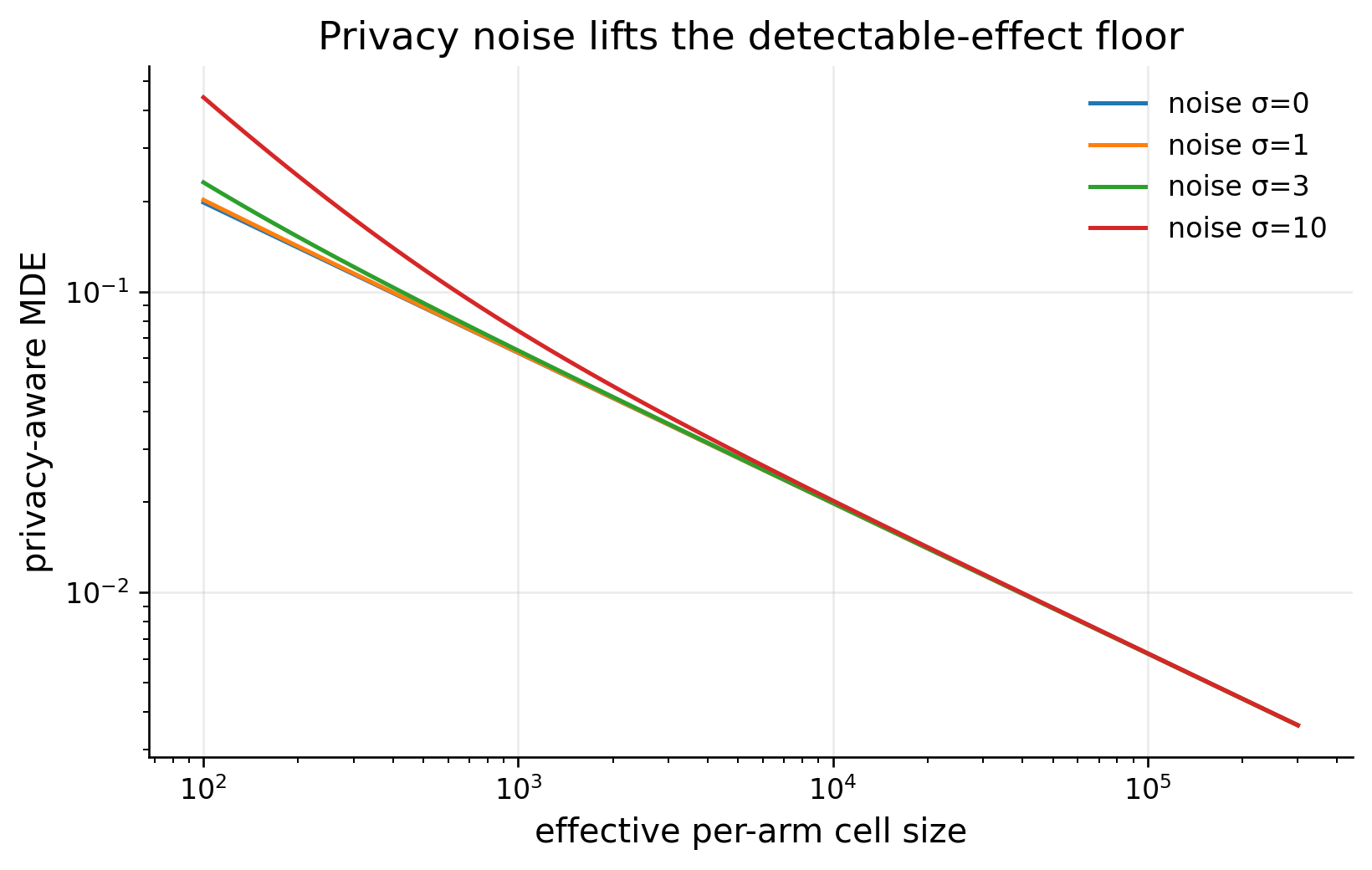}
  \caption{Privacy-aware minimum detectable effects. Reporting noise increases the detectable-effect floor most strongly for small effective cell sizes, while large cells dilute the noise contribution.}
  \label{fig:mde-appendix}
\end{figure}

\subsection{Heterogeneous signal-loss reversal}

Figure~\ref{fig:heterogeneous-appendix} visualizes the constructive example in Appendix~\ref{ex:heterogeneous-signal-loss}. The construction contains two segments with equal weights and clean control mean zero. In world \(u\), segment 1 has clean treated mean \(1\) and retention \(0.5\), while segment 2 has clean treated mean \(0\). In world \(u'\), the segment roles are reversed. Both worlds produce the same aggregate degraded treated mean \(0.25\). The appendix diagnostic then applies the same match-loss lower-bound rule used by the main certification implementation. At threshold \(b=0.25\), world \(u\) gives segment 1 band \([0.5,1.0]\) and segment 2 band \([0,0]\), so segment 1 certifies and segment 2 rejects. World \(u'\) reverses these bands and certifies segment 2 instead. This example explains why aggregate privacy-degraded lift should not be interpreted as segment safety when signal loss may vary across segments.

\begin{figure}[t]
  \centering
  \includegraphics[width=\linewidth]{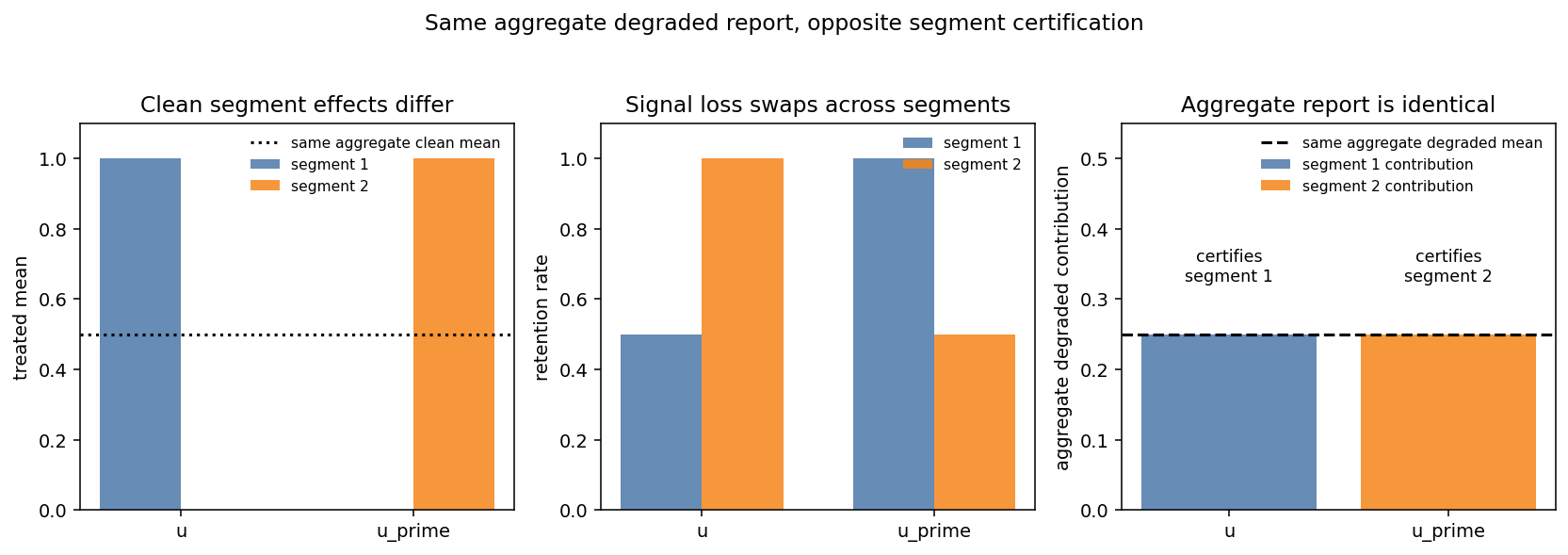}
  \caption{Heterogeneous signal loss can reverse segment decisions. Two clean segment-level worlds induce the same aggregate degraded report, while the lower-bound certification rule certifies different segments.}
  \label{fig:heterogeneous-appendix}
\end{figure}

\section{Supporting Results and Proofs}

\begin{lemma}[Fiber projection is the sharp identified set]
\label{lem:fiber-projection}
Fix a segment \(s\), a released report \(r_s\), an ambiguity set \(\U\), and the compatibility fiber \(\F_s(r_s;\U)\) in Eq.~\eqref{eq:compatibility-fiber}. Suppose \(\F_s(r_s;\U)\) is nonempty. The set of incrementality values that can be induced by clean worlds compatible with the report is
\[
\{\Delta_s(\eta_s):\eta_s\in\F_s(r_s;\U)\}.
\]
Its closed convex hull is the interval
\[
[\ell_s^\star(r_s;\U),u_s^\star(r_s;\U)]
\]
defined in Eq.~\eqref{eq:fiber-projection}. No uniformly valid procedure that observes only \(r_s\) can replace this interval by a strict subinterval without excluding a compatible clean world.
\end{lemma}

\begin{proof}
By definition, \(\eta_s\in\F_s(r_s;\U)\) if and only if there is at least one signal-loss mechanism \(u\in\U\) under which the clean world \(\eta_s\) can generate the released report \(r_s\). Therefore the set of clean incrementality values that cannot be ruled out from \(r_s\) is exactly
\[
\mathcal D_s(r_s;\U)
=
\{\Delta_s(\eta_s):\eta_s\in\F_s(r_s;\U)\}.
\]
The smallest closed interval containing this set is
\[
\left[
\inf_{\eta_s\in\F_s(r_s;\U)}\Delta_s(\eta_s),
\sup_{\eta_s\in\F_s(r_s;\U)}\Delta_s(\eta_s)
\right],
\]
which is Eq.~\eqref{eq:fiber-projection}. This proves that the fiber projection is an outer bound on all compatible clean effects.

It remains to show sharpness. Suppose a procedure observing only \(r_s\) reports a strict subinterval \(I\) of this projection interval while claiming uniform validity over \(\U\). Because \(I\) is strict, there is some value \(\delta\in\mathcal D_s(r_s;\U)\) outside \(I\), or there is a sequence of compatible values converging to a boundary point outside \(I\). In the first case, choose \(\eta_s^\circ\in\F_s(r_s;\U)\) with \(\Delta_s(\eta_s^\circ)=\delta\). The report \(r_s\) is compatible with \(\eta_s^\circ\), so excluding \(\delta\) violates validity in that clean world. In the boundary case, the same conclusion follows by taking a compatible value sufficiently close to the excluded boundary. Hence no uniformly valid report-based procedure can improve on the fiber projection without additional assumptions.
\end{proof}

\subsection{Proof of Theorem~\ref{thm:sharp-frontier}}

\begin{proof}
Fix segment \(s\) and arm \(a\). The degraded mean satisfies
\[
\widetilde\mu_{a,s}=q_{a,s}\mu_{a,s},
\qquad
q_{a,s}\in[\underline q_{a,s},\overline q_{a,s}],
\qquad
0\le \mu_{a,s}\le B.
\]
For any feasible retention rate \(q_{a,s}\), the clean mean must be
\[
\mu_{a,s}=\widetilde\mu_{a,s}/q_{a,s}.
\]
Since \(q_{a,s}\mapsto \widetilde\mu_{a,s}/q_{a,s}\) is decreasing on the positive interval \([\underline q_{a,s},\overline q_{a,s}]\), the smallest feasible clean mean is obtained at \(q_{a,s}=\overline q_{a,s}\), and the largest feasible clean mean is obtained at \(q_{a,s}=\underline q_{a,s}\), subject to the outcome bound \(B\). Thus the arm-specific feasible interval is
\[
\mathcal I_{a,s}
=
\left[
\frac{\widetilde\mu_{a,s}}{\overline q_{a,s}},
\min\left\{B,\frac{\widetilde\mu_{a,s}}{\underline q_{a,s}}\right\}
\right].
\]
Every point in this interval is attainable, which is the step that proves sharpness rather than only validity of the outer bound. If \(\mu\in\mathcal I_{a,s}\) and \(\widetilde\mu_{a,s}>0\), choose \(q=\widetilde\mu_{a,s}/\mu\), which lies in the retention interval by construction. Then construct a bounded clean outcome with mean \(\mu\), for example \(Y(a)=B\) with probability \(\mu/B\) and \(Y(a)=0\) otherwise. Let the outcome be reported with probability \(q\), independently conditional on the arm and segment. The degraded mean is then \(q\mu=\widetilde\mu_{a,s}\). If \(\widetilde\mu_{a,s}=0\), then \(\mathcal I_{a,s}=\{0\}\) because retention is bounded away from zero, and choosing any admissible \(q\) with \(\mu=0\) attains the interval. Thus \(\mathcal I_{a,s}\) is exactly the set of clean arm means compatible with the degraded arm mean, not merely a conservative superset.

The treatment and control arm means can be chosen independently within \(\mathcal I_{1,s}\) and \(\mathcal I_{0,s}\). Therefore the identified set for \(\Delta_s=\mu_{1,s}-\mu_{0,s}\) is the Minkowski difference
\[
\mathcal I_{1,s}-\mathcal I_{0,s}
=
\left[
\inf_{\mu_1\in\mathcal I_{1,s},\,\mu_0\in\mathcal I_{0,s}}(\mu_1-\mu_0),
\sup_{\mu_1\in\mathcal I_{1,s},\,\mu_0\in\mathcal I_{0,s}}(\mu_1-\mu_0)
\right],
\]
which is exactly \([\ell_s(\U),u_s(\U)]\) in Eq.~\eqref{eq:match-loss-band}. Any \(\delta\) inside this interval can be written as \(\delta=\mu_1-\mu_0\) for some \(\mu_1\in\mathcal I_{1,s}\) and \(\mu_0\in\mathcal I_{0,s}\), so the interval is sharp.

The first two decision statements follow immediately from this sharpness. If \(\ell_s(\U)>b_s\), every compatible clean world has \(\Delta_s>b_s\), so positive incrementality is certifiable uniformly over \(\U\). If \(u_s(\U)\le b_s\), every compatible clean world has \(\Delta_s\le b_s\), so non-incrementality relative to \(b_s\) is rejectable uniformly over \(\U\).

It remains to prove the lower-bound statement for the unresolved region. Suppose \(\ell_s(\U)\le b_s<u_s(\U)\). By sharpness, there exists a compatible clean world \(\eta_-\) and a mechanism \(u_-\in\U\) with \(\Delta_s(\eta_-)\le b_s\) that produces the observed degraded arm means. Likewise, because \(b_s<u_s(\U)\), there exists a compatible clean world \(\eta_+\) and a mechanism \(u_+\in\U\) with \(\Delta_s(\eta_+)>b_s\) that produces the same degraded arm means. Let \(P_-\) and \(P_+\) denote the corresponding released-report distributions. In the deterministic degraded-report case these distributions are identical point masses at the observed report. With randomized reporting noise, take the same released noisy report distribution conditional on the same degraded arm means, so \(P_-=P_+\) as well. Thus the released report contains no information that distinguishes the below-or-at-threshold world from the above-threshold world.

Now consider any binary rule \(\varphi\) mapping released reports to \(\{\mathrm{certify},\mathrm{not\ certify}\}\). Let \(A=\{r:\varphi(r)=\mathrm{certify}\}\). Under \(P_-\), certification is a false positive, so the false-certification probability is \(P_-(A)\). Under \(P_+\), non-certification is a false negative, so the false-noncertification probability is \(P_+(A^c)\). Since \(P_+=P_-\),
\[
P_-(A)+P_+(A^c)=P_-(A)+P_-(A^c)=1.
\]
Therefore
\[
\max\{P_-(A),P_+(A^c)\}\ge \frac{1}{2}.
\]
Equivalently, every binary rule has worst-case error at least \(1/2\). More generally, the same argument with two nonidentical report distributions gives the lower bound
\[
\max\left\{
P_+(\varphi\ne\mathrm{certify}),
P_-(\varphi=\mathrm{certify})
\right\}
\ge
\frac{1-\TV(P_+,P_-)}{2},
\]
because \(P_+(A)-P_-(A)\le \TV(P_+,P_-)\) for every measurable certification set \(A\). Finally, uniform validity rules out certification for this report because \(\eta_-\) is compatible and has \(\Delta_s\le b_s\). Uniform validity also rules out rejection because \(\eta_+\) is compatible and has \(\Delta_s>b_s\). The only uniformly valid output is therefore unresolved.
\end{proof}

\subsection{Proof of Proposition~\ref{prop:sample-complexity}}

\begin{proof}
Fix an arm \(a\) and segment \(s\). Let
\[
\widetilde\mu_{a,s}
=
\E[\widetilde Y_i\mid A_i=a,i\in s]
\]
be the population degraded arm mean, and let
\[
\widehat{\widetilde\mu}_{a,s}
=
\bar{\widetilde Y}_{a,s}
+
\frac{\xi_{a,s}}{n_s}
\]
be the released degraded arm mean after randomized reporting noise. Here \(\bar{\widetilde Y}_{a,s}\) is the empirical degraded mean before reporting noise, \(n_s\ge n\) is the number of randomized units in the arm-segment cell, and \(\xi_{a,s}\) is the noise added to the released arm total.

The released mean error decomposes as
\[
\widehat{\widetilde\mu}_{a,s}-\widetilde\mu_{a,s}
=
\left(\bar{\widetilde Y}_{a,s}-\widetilde\mu_{a,s}\right)
+
\frac{\xi_{a,s}}{n_s}.
\]
Since \(0\le \widetilde Y_i\le B\), Hoeffding's lemma implies that \(\bar{\widetilde Y}_{a,s}-\widetilde\mu_{a,s}\) is sub-Gaussian with variance proxy \(B^2/(4n_s)\). The reporting-noise term \(\xi_{a,s}/n_s\) is sub-Gaussian with variance proxy \(\sigma_\xi^2/n_s^2\). Because the randomized reporting noise is independent of the sample draw, their sum is sub-Gaussian with variance proxy
\[
v_{a,s}
=
\frac{B^2}{4n_s}
+
\frac{\sigma_\xi^2}{n_s^2}
\le
\frac{B^2}{4n}
+
\frac{\sigma_\xi^2}{n^2}.
\]
Let \(L=\log(4|\C|/\alpha)\). The sub-Gaussian tail bound gives
\[
\Pbb\left(
\left|
\widehat{\widetilde\mu}_{a,s}-\widetilde\mu_{a,s}
\right|
>
\sqrt{2v_{a,s}L}
\right)
\le
2e^{-L}
=
\frac{\alpha}{2|\C|}.
\]
There are \(2|\C|\) arm-segment means, so a union bound implies simultaneous control over both arms and all segments with probability at least \(1-\alpha\). Moreover,
\[
\sqrt{2v_{a,s}L}
\le
B\sqrt{\frac{L}{2n}}
+
\frac{\sigma_\xi}{n}\sqrt{2L}
\equiv
r_n,
\]
where the last inequality uses \(n_s\ge n\) and \(\sqrt{x+y}\le\sqrt{x}+\sqrt{y}\). Therefore, on the simultaneous event,
\[
\left|
\widehat{\widetilde\mu}_{a,s}-\widetilde\mu_{a,s}
\right|
\le r_n
\]
for every arm and segment.

The clean arm mean is obtained from a degraded mean by division by an admissible retention rate. Because every admissible retention rate is at least \(q_{\min}\), an arm-level degraded-mean error of size \(r_n\) changes any compatible clean arm mean by at most \(r_n/q_{\min}\). The treatment effect is a difference of two arm means, so the finite-sample expansion of the population signal-loss band has radius
\[
\Rad_n
=
\frac{2r_n}{q_{\min}}
=
\frac{2B}{q_{\min}}\sqrt{\frac{\log(4|\C|/\alpha)}{2n}}
+
\frac{2\sigma_\xi}{q_{\min}n}\sqrt{2\log(4|\C|/\alpha)}.
\]
This proves Eq.~\eqref{eq:sample-complexity-error}.

To obtain the displayed sufficient sample size, require each term in \(\Rad_n\) to be at most \(\varepsilon/2\). The sampling term is at most \(\varepsilon/2\) when
\[
n
\ge
\frac{8B^2}{q_{\min}^2\varepsilon^2}
\log\frac{4|\C|}{\alpha}.
\]
The reporting-noise term is at most \(\varepsilon/2\) when
\[
n
\ge
\frac{4\sigma_\xi}{q_{\min}\varepsilon}
\sqrt{2\log\frac{4|\C|}{\alpha}}.
\]
Combining these two sufficient inequalities gives Eq.~\eqref{eq:sample-complexity-n} up to universal constants. Finally, \([\ell_s(\U),u_s(\U)]\) is the population compatibility band from signal-loss uncertainty. The proof above only adds and removes finite-sample radius around this band; the identification width itself remains. Thus \(u_s(\U)-\ell_s(\U)\) persists even as \(n\to\infty\) unless the retention ambiguity set tightens.
\end{proof}

\subsection{Proof of Theorem~\ref{thm:minimax-signal-loss}}

\begin{proof}
We first prove the sampling lower bound when the retention probability \(q\) is known. A lower bound for a restricted subproblem is also a lower bound for the full model, so it is enough to hold the control arm fixed and vary only the treated arm. Set \(p_0=1/2\). The target becomes
\[
\Delta=B(p_1-p_0)=B\left(p_1-\frac12\right).
\]
For treated units, \(Y_i(1)/B\sim\mathrm{Bernoulli}(p_1)\) and \(R_i(1)\sim\mathrm{Bernoulli}(q)\) independently. Therefore the normalized degraded observation
\[
\frac{\widetilde Y_i(1)}{B}
=
\frac{R_i(1)Y_i(1)}{B}
\]
is Bernoulli with success probability \(q p_1\). Signal loss has therefore reduced the effective Bernoulli success probability by a factor \(q\), which is where the information loss enters.

Choose two treated-arm conversion probabilities
\[
p_1^-=\frac12,
\qquad
p_1^+=\frac12+h,
\]
where \(0<h\le 1/4\). Both probabilities lie in \([1/4,3/4]\). The corresponding clean incrementality values are
\[
\Delta^- =0,
\qquad
\Delta^+ = Bh,
\]
so their separation is \(Bh\). Let \(P_-\) and \(P_+\) denote the laws of the \(n\) degraded treated-arm observations under these two worlds. Then
\[
P_-=\mathrm{Bernoulli}(q/2)^{\otimes n},
\qquad
P_+=\mathrm{Bernoulli}\!\left(q\left(\frac12+h\right)\right)^{\otimes n}.
\]

We next choose \(h\) so these two clean effects are separated, while the degraded experiments remain hard to distinguish. For Bernoulli probabilities \(r\) and \(r+\delta\), with \(0<r<r+\delta<1\),
\[
\mathrm{KL}\!\left(\mathrm{Bernoulli}(r)\,\|\,\mathrm{Bernoulli}(r+\delta)\right)
\le
\frac{\delta^2}{(r+\delta)(1-r-\delta)}.
\]
Here \(r=q/2\) and \(\delta=qh\). Because \(h\le 1/4\),
\[
r+\delta
=
q\left(\frac12+h\right)
\ge
\frac q2,
\qquad
1-r-\delta
=
1-q\left(\frac12+h\right)
\ge
\frac14,
\]
where the second inequality uses \(q\le 1\). Hence \((r+\delta)(1-r-\delta)\ge q/8\), and the one-observation KL divergence is at most \(8q h^2\). Tensorization of KL divergence over \(n\) independent observations gives
\[
\mathrm{KL}(P_-\,\|\,P_+)
\le
8nqh^2.
\]
Choose
\[
h=\frac{c_0}{\sqrt{qn}}
\]
with \(c_0>0\) sufficiently small and \(qn\) sufficiently large so that \(h\le 1/4\) and \(\mathrm{KL}(P_-\,\|\,P_+)\le 1/8\). Pinsker's inequality then gives
\[
\TV(P_-,P_+)
\le
\sqrt{\frac{\mathrm{KL}(P_-\,\|\,P_+)}{2}}
\le
\frac14.
\]

Let \(\widehat\Delta\) be any estimator based on the degraded observations. The expectation form of Le Cam's two-point argument follows from the identity
\[
\int \min\{dP_-,dP_+\}
=
1-\TV(P_-,P_+).
\]
For every possible estimator value \(t\),
\[
|t-\Delta^-|+|t-\Delta^+|
\ge
|\Delta^+-\Delta^-|
=
Bh.
\]
Integrating this inequality with respect to the common part \(\min\{dP_-,dP_+\}\) yields
\[
\E_-|\widehat\Delta-\Delta^-|
+
\E_+|\widehat\Delta-\Delta^+|
\ge
Bh\{1-\TV(P_-,P_+)\}.
\]
Therefore
\[
\sup_{\theta\in\{-,+\}}
\E_\theta|\widehat\Delta-\Delta_\theta|
\ge
\frac{Bh}{2}\{1-\TV(P_-,P_+)\}
\ge
\frac{3Bh}{8}.
\]
Substituting \(h=c_0/\sqrt{qn}\) gives
\[
\sup_{p_0,p_1\in[1/4,3/4]}
\E_{p_0,p_1}|\widehat\Delta-B(p_1-p_0)|
\ge
c\frac{B}{\sqrt{qn}}
\]
for a universal constant \(c>0\). This proves the stated minimax estimation-rate lower bound.

The sample-size form follows from the same construction. If an estimator had absolute error at most \(\varepsilon\) with high constant probability uniformly over the class, then, whenever \(\varepsilon<Bh/2\), the intervals \([\Delta^- -\varepsilon,\Delta^-+\varepsilon]\) and \([\Delta^+-\varepsilon,\Delta^++\varepsilon]\) would be disjoint. Thresholding \(\widehat\Delta\) at \((\Delta^-+\Delta^+)/2\) would then produce a test that distinguishes \(P_-\) from \(P_+\) with the same high constant success probability. The best possible success probability for testing two simple hypotheses with equal prior is \((1+\TV(P_-,P_+))/2\), which is at most \(5/8\) under the construction above. Thus constant-probability estimation error \(\varepsilon\) requires \(\varepsilon\gtrsim Bh\). Since \(h=c_0/\sqrt{qn}\), this implies
\[
n
\gtrsim
\frac{B^2}{q\varepsilon^2},
\]
which gives the displayed lower bound up to a universal constant.

We now prove the non-vanishing ambiguity lower bound when the retention probability is not known exactly. This part is an identification argument rather than a sampling argument, so it remains even with infinite data. Consider a one-arm subproblem with the control mean fixed at zero. Let the observed degraded treated mean be
\[
\widetilde\mu=\frac{\underline q B}{2}.
\]
This single degraded mean is compatible with two different clean worlds. In the first world, the retention probability is \(q=\underline q\), so the compatible clean treated mean is
\[
\mu^+
=
\frac{\widetilde\mu}{\underline q}
=
\frac B2.
\]
In the second world, the retention probability is \(q=\overline q\), so the compatible clean treated mean is
\[
\mu^-
=
\frac{\widetilde\mu}{\overline q}
=
\frac B2\frac{\underline q}{\overline q}.
\]
Both clean means lie in \([0,B]\), both retention probabilities lie in the ambiguity set \([\underline q,\overline q]\), and both worlds induce exactly the same degraded report because \(q\mu=\widetilde\mu\) in both cases. Their clean incrementality values differ by
\[
\mu^+-\mu^-
=
\frac B2\left(1-\frac{\underline q}{\overline q}\right).
\]
Any estimator that observes only the degraded report must return the same numerical value in these two worlds. For any number \(t\),
\[
\max\{|t-\mu^+|,\ |t-\mu^-|\}
\ge
\frac{\mu^+-\mu^-}{2}.
\]
Consequently the worst-case absolute error over these two compatible clean worlds is at least
\[
\frac B4\left(1-\frac{\underline q}{\overline q}\right).
\]
This proves the non-vanishing identification lower bound.
\end{proof}

\begin{proposition}[Privacy-aware MDE derivation]
\label{prop:privacy-mde}
For a two-arm randomized experiment with assignment-unit variance proxy \(\sigma_s^2\), effective per-arm cell size \(N_s\), and independent reporting-noise variance \(\sigma_\xi^2\) added to each arm total, the normal-approximation minimum detectable effect is Eq.~\eqref{eq:privacy-mde}.
\end{proposition}

\subsection{Proof of Proposition~\ref{prop:privacy-mde}}

\begin{proof}
Let \(\widehat\Delta_s^{\mathrm{priv}}\) denote the reported treated-control difference in segment \(s\). Under a balanced effective per-arm sample size \(N_s\), the randomization variance proxy for a difference in means is
\[
\mathrm{Var}_{\mathrm{rand}}(\widehat\Delta_s)
\approx
\frac{\sigma_s^2}{N_s}+\frac{\sigma_s^2}{N_s}
=
\frac{2\sigma_s^2}{N_s}.
\]
If the reported treated and control totals include independent mean-zero randomized reporting noises \(\xi_{1s}\) and \(\xi_{0s}\) with variance \(\sigma_\xi^2\), then the reported means contain additional noise terms \(\xi_{1s}/N_{1s}\) and \(\xi_{0s}/N_{0s}\). Because these noises are independent and enter the difference with opposite signs, their variance contribution is
\[
\sigma_\xi^2\left(\frac{1}{N_{1s}^2}+\frac{1}{N_{0s}^2}\right).
\]
Adding the randomization and reporting-noise contributions gives the variance inside Eq.~\eqref{eq:privacy-mde}. For a normal-approximation power calculation, a two-sided level-\(\alpha\) test rejects when the estimated effect exceeds \(z_{1-\alpha/2}\) standard errors under the null. To have power \(1-\beta\) against an alternative effect size \(d\), the mean shift \(d\) must exceed this rejection threshold by another \(z_{1-\beta}\) standard errors. Thus, with \(\mathrm{SE}\) denoting standard error,
\[
d
=
\left(z_{1-\alpha/2}+z_{1-\beta}\right)\mathrm{SE}.
\]
Substituting the privacy-aware standard error gives Eq.~\eqref{eq:privacy-mde}. The sample-size statement follows because certification requires the statistically detectable effect to be no larger than the robust population margin \(\ell_s(\U)-b_s\); solving that inequality over feasible allocations gives the required sample size.
\end{proof}

\subsection{Heterogeneous signal-loss reversal example}
\label{ex:heterogeneous-signal-loss}

This example shows that aggregate privacy-degraded lift can be compatible with opposite segment-level certification decisions. Let outcomes be bounded by \(B=1\), let the two segments have equal population weight, and let the clean control mean be zero in both segments. Let the ambiguity set allow segment-specific treated-arm retention rates in \([1/2,1]\); control retention is irrelevant because the clean control means are zero.

Under mechanism \(u\), set the clean treated means to
\[
\mu_{1,1}=1,
\qquad
\mu_{1,2}=0,
\]
and choose treated retention rates
\[
q_{1,1}=1/2,
\qquad
q_{1,2}=1.
\]
The degraded treated means are \(\widetilde\mu_{1,1}=1/2\) and \(\widetilde\mu_{1,2}=0\), so the aggregate degraded treated mean is \(1/4\). The aggregate degraded control mean is zero.

Under mechanism \(u'\), reverse the segment roles:
\[
\mu'_{1,1}=0,
\qquad
\mu'_{1,2}=1,
\]
with treated retention rates
\[
q'_{1,1}=1,
\qquad
q'_{1,2}=1/2.
\]
The degraded treated means are \(\widetilde\mu'_{1,1}=0\) and \(\widetilde\mu'_{1,2}=1/2\), so the aggregate degraded treated mean is again \(1/4\), and the aggregate degraded control mean is again zero. Thus the two worlds have the same aggregate degraded lift.

The segment-level lower-bound decisions are opposite. Under \(u\), the segment 1 degraded treated mean is \(1/2\), its control mean is zero, and the largest admissible retention rate is \(1\), so the segment-level lower bound is at least \(1/2\). Segment 2 has degraded treated and control means equal to zero, so both its lower and upper bounds are zero. Thus for any \(b\in(0,1/2)\), the lower-bound rule certifies segment 1 and rejects segment 2. Under \(u'\), the same argument certifies segment 2 and rejects segment 1. The aggregate degraded report is unchanged across the two constructions. The example shows that aggregate certification cannot imply segment-level certification unless \(\U\) restricts how signal loss can vary across segments.

\end{document}